\pdfoutput=1
\documentclass[11pt]{article}
\usepackage[english]{babel} 
\usepackage[T1]{fontenc} 
\usepackage[utf8]{inputenc} 
\usepackage[babel]{microtype} 

\usepackage{comment}
\usepackage{authblk}
\usepackage{lmodern}
\usepackage{bbm}
\usepackage{url}
\usepackage{csquotes}
\usepackage{amsmath,amssymb,amsthm,bm,mathtools}
\usepackage[pdftex]{graphicx}
\usepackage{enumitem}
\usepackage{booktabs}
\usepackage[all]{nowidow}  
\usepackage{siunitx}
\usepackage{color}
\usepackage{caption}
\usepackage{subcaption}
\usepackage[dvipsnames]{xcolor}
\usepackage{mdframed}
\usepackage{lipsum}
\usepackage{natbib}
\usepackage{wrapfig}
\usepackage{setspace}
\usepackage[margin=1in]{geometry}
\usepackage{tcolorbox}
\tcbuselibrary{skins}
\usepackage{framed}
\usepackage{parskip}

\usepackage{tikz}
\usetikzlibrary{arrows,automata,calc}

\usepackage{hyperref}
\hypersetup{colorlinks=true, citecolor=MidnightBlue, linkcolor=MidnightBlue, urlcolor=MidnightBlue}

\usepackage{algorithm}
\usepackage{algpseudocodex}

\newlist{enuminline}{enumerate*}{1}
\setlist[enuminline,1]{label=\itshape\alph*\upshape)}

\newcommand{\Email}[1]{\href{mailto:#1}{\nolinkurl{#1}}}

\newtheorem{example}{Example}
\newtheorem{remark}{Remark}
\newtheorem{lemma}{Lemma}
\newtheorem*{lemma*}{Lemma}
\newtheorem{theorem}{Theorem}
\newtheorem{assumption}{Assumption}
\newtheorem{definition}{Definition}
\newtheorem{cor}{Corollary}

\DeclareMathOperator*{\argmax}{arg\,max}
\DeclareMathOperator*{\argmin}{arg\,min}
\DeclareMathOperator*{\Prob}{\mathbb{P}}
\DeclareMathOperator*{\E}{\mathbb{E}}

\newcommand{\Dscr}{\ensuremath{\mathcal D}}

\newcommand{\Rscr}{\ensuremath{\mathcal R}}

\newcommand{\Abb}{\ensuremath{\mathbb A}}

\newcommand{\Hbb}{\ensuremath{\mathbb H}}

\newcommand{\Nbb}{\ensuremath{\mathbb N}}

\newcommand{\Rbb}{\ensuremath{\mathbb R}}
\newcommand{\Sbb}{\ensuremath{\mathbb S}}

\newcommand{\Xbb}{\ensuremath{\mathbb X}}

\newcommand{\Zbb}{\ensuremath{\mathbb Z}}
\newcommand{\ind}{\ensuremath{\mathbbm 1}}
 \mdfdefinestyle{MyFrame}{%
	linecolor=black,
	outerlinewidth=2pt,
	innertopmargin=4pt,
	innerbottommargin=4pt,
	innerrightmargin=4pt,
	innerleftmargin=4pt,
	leftmargin = 4pt,
	rightmargin = 4pt
}

\colorlet{responsecolor}{Blue}
\definecolor{shadecolor}{gray}{0.95}  

\title{Optimizing Audio Recommendations for the Long-Term:\\ A Reinforcement Learning Perspective}

\author[1]{Lucas Maystre}
\author[1,2]{Daniel Russo}
\author[1]{Yu Zhao}
\affil[1]{\footnotesize Spotify}
\affil[2]{\footnotesize Columbia University}
\date{\today}

\begin{document}




\maketitle
\begin{abstract}
We present a novel podcast recommender system deployed at industrial scale. This system successfully optimizes personal listening journeys that unfold over months for hundreds of millions of listeners. In deviating from the pervasive industry practice of optimizing machine learning algorithms for short-term proxy metrics, the system substantially improves long-term performance in A/B tests. The paper offers insights into how our methods cope with attribution, coordination, and measurement challenges that usually hinder such long-term optimization. To contextualize these practical insights within a broader academic framework, we turn to reinforcement learning (RL). Using the language of RL, we formulate a comprehensive model of users' recurring relationships with a recommender system. Then, within this model, we identify our approach as a policy improvement update to a component of the existing recommender system, enhanced by tailored modeling of value functions and user-state representations. Illustrative offline experiments suggest this specialized modeling reduces data requirements by as much as a factor of 120,000 compared to black-box approaches.
\end{abstract}

\section{Introduction}
Recommendation systems are an essential component of modern online platforms. They help individuals find candidates to interview for job postings, dating partners, products to try, or media to engage with. These individuals, called ``users'', typically have recurring relationships with recommendation systems, interacting with them repeatedly over extended periods. The systems themselves are generally powered by highly sophisticated machine learning algorithms. However, due to the challenges of measuring long-term outcomes, attributing them to specific recommendations, and coordinating improvements across large, decentralized systems, these algorithms are almost invariably trained to optimize short-term metrics. They are then deployed to govern recurring user interactions, despite the mismatch between their optimization criteria and the long-term nature of user engagement. This mismatch is widely recognized \citep[see e.g.][]{wu2017returning}.

 At Spotify, a leading audio streaming service, we have implemented a novel approach to podcast recommendations that successfully optimizes for outcomes of long-term user journeys. This approach has shown immense impact, now powering recommendations for hundreds of millions of users worldwide. In one A/B test, our method increased average listening time attributable to recommendations by 81\% for affected users over a 60-day period. A larger-scale experiment demonstrated significant improvements in overall app-level outcomes, despite altering only a small component of the app. The magnitude of improvement, compared to myopic approaches pervasive in industry, represents a significant learning. It is suggestive of massive untapped potential of optimizing for long-term user satisfaction in recommendation systems, more broadly.

This paper not only describes our real-world implementation at Spotify and  shares key learnings from both online and offline testing, but also provides a rigorous interpretation by drawing precise connections to reinforcement learning (RL). RL gives a formal language for studying the problem of learning across users to optimize recurring interactions with individual users (See Figure \ref{fig:bandits-vs-rl}). Our formulation focuses on an ``offline RL'' problem, in which the objective is to use historical data to implement a policy improvement update\footnote{A policy improvement update is also known as a policy iteration update.}  to a single component of the recommendation policy. This component represents specialized logic that powers certain podcast recommendations. In the (idealized) problem setting displayed in Figure \ref{fig:bandits-vs-rl}, one could imagine gathering data on repeated interactions with past users and deploying the resulting policy to govern interactions with future users. 

The key to policy improvement is the estimation of a state-action value function, dubbed a ``$Q$-function''. In this context, the $Q$-function quantifies how short-term deviations from the incumbent recommendation policy impact the ``long-term expected reward'' associated with a user's future app interactions. 
\begin{figure}[t!]
    \includegraphics[width=\textwidth]{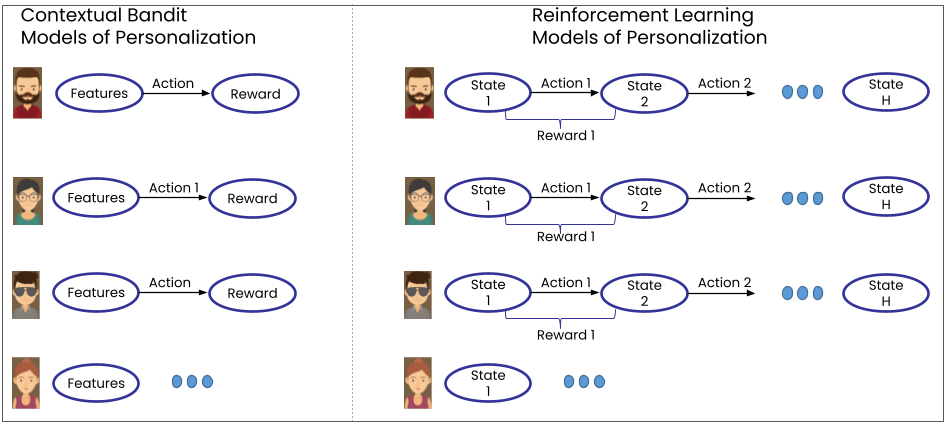}
    \caption{Explaining RL models of personalization, by contrast with contextual bandit models: In both types of models, the goal is to learn from interacting with users how to optimize interactions with future users. Contextual bandit algorithms learn to optimize the immediate reward accrued from a recommendation action \citep{li2010contextual}. RL models aim to optimize a sequence of interactions with an individual user, acknowledging that recommendation decisions in one period can impact the efficacy of recommendations in future periods.}
    \label{fig:bandits-vs-rl}
\end{figure}

Our approach to modeling the $Q$-function is driven by a key hypothesis: recommendations significantly contribute to long-term user satisfaction by fostering the formation of specific, recurring engagement patterns with individual pieces of content, which we call ``item-level listening habits''. In the context of audio streaming, these habits can manifest in various forms: a user might develop a regular routine with a specific podcast show, returning to it as new episodes are released; form an attachment to a particular creator's content; or integrate a curated playlist into their daily activities. Here, we use the term ``item'' to refer to any distinct piece of content that can be recommended, such as a podcast show\footnote{A nuance is that, in some of our implementations, the recommendations advertise specific podcast episodes, but the habits we subsequently track are with the parent show.}, an album, or a playlist. For instance, consider the hypothetical 60-day user journey in Figure \ref{fig:user-journey}. The user's journey with the app as a whole is quite complex, but repeated interactions with the show ``Podcast X'' seem intrinsically linked. In this fictitious journey, individual recommendations not only lead to immediate engagement but also help a user discover a podcast show they return to regularly, becoming part of their routine.  Collectively, these individual habits coalesce to define the user's overall engagement pattern with the platform. This hypothesis underpins our entire modeling approach and distinguishes it from more generic RL approaches.
\begin{figure}\centering
    \includegraphics[width=6in]{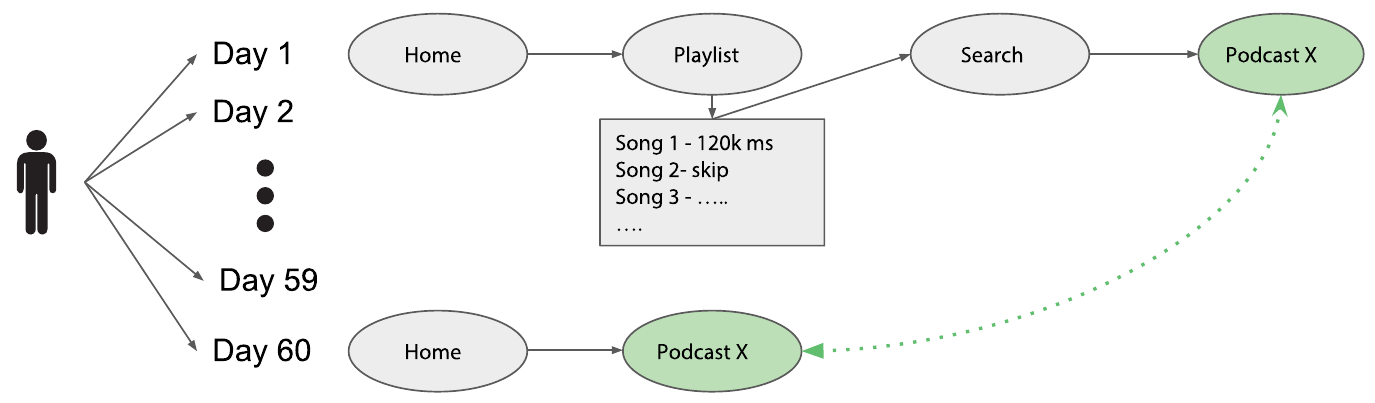}
    \caption{A depiction of a possible trajectory of user interactions over 60 days. Highlighted in green on the first day is an interaction when the user searches for the term `podcast' and receives personalized recommendations through which  they discover \emph{Podcast X}, a hypothetical podcast show that releases new episodes on a regular cadence. Subsequent recommendations resurface the show, and sixty days later the user has formed a deep connection and is still listening to \emph{Podcast X}. }
    \label{fig:user-journey}
\end{figure}

Motivated by this hypothesis, we tailor $Q$-function modeling to capture how a recommendation of a specific item, to a specific user, helps foster a satisfying listening habit with the item. This modeling involves three new contributions:
\begin{enumerate}
    \item We introduce a pragmatic user-state representation that is tailored to this hypothesis. A novel and  critical component of this representation is something we call a ``content-relationship state'', offering a compressed encoding of the user's listening history with any particular piece of content (called an ``item'').
    \item We introduce item-level value functions, dubbed `stickiness models'. These stickiness models project the strength of a user's long-term listening habit with a specific item.
    \item We introduce structural assumptions under which the $Q$-function decomposes conveniently, enabling us to break the complex task of long-term optimization into more manageable, item-specific predictions. In particular, the long-term benefit of a recommendation is broken up into a) modeling the user's direct response to the recommendation in the short-term, b) calculating the transition in the users' content-relationship state with that item following different levels of immediate engagement, and c) calculating how this state transition changes stickiness predictions for that user and \emph{that specific item}.
\end{enumerate}
Rather than directly making predictions about complex app-level interactions that are only loosely connected to the recommendation, all components of this logic focus on future engagement with the recommended item itself. In practice, we rely on pre-existing, highly optimized models that predict short-term engagement based on the system's understanding of the user's unique tastes, the item's unique appeal, and the recommendation context. The logic above allows to augment these pre-existing models with newly trained stickiness models. Intuitively, this augmentation proves to be so impactful due to a crucial mismatch: the podcast shows a user is most likely to try (i.e., listen to once) are often not the ones they are most likely to subsequently stick with (i.e., listen to on a recurring basis). Hence, optimizing for the long-term produces differentiated recommendations, resulting in much better long-term performance.

\subsection{Broader insights into the challenges solutions must overcome}
\label{subsec:challenges}
While our methodology has proven highly impactful, optimizing for long-term outcomes in large-scale recommendation systems presents significant challenges. Section \ref{subsec:challenges-revisited} offers insights into how our implementations succeed in the face of notable attribution, coordination, and measurement challenges.

\begin{description}
\item[Measurement.]
A poor signal-to-noise ratio makes measurement challenging.
Consider some very holistic outcome of the 60 day user-journey, like whether the user remains a subscriber after 60 days, or the total number of minutes they listen.
One approach to unbiased measurement is to inject randomness into recommendation decisions and
then assess whether users who received certain recommendations by chance tend to have ``better'' outcomes.
The signal-to-noise ratio is very poor, however, both because user-behavior is inherently noisy and because a single recommendation is a tiny part of their overall 60 day experience.

\item[Attribution.]
After observing a positive 60 day trajectory, it is difficult to assign credit to individual actions in a coherent way. Should a user's experience listening to \emph{Podcast X} be credited to the recommendation that first led to the discovery, to decisions to recommend that item again later, or, more ambitiously, to the personalized playlist that preceded its discovery?
A subtlety is that whether a decision is effective at this point in time may depend on how decisions will be made at other points in the future.
Recommending the user try \emph{Podcast X} for the first time might only be highly valuable if future recommendations are likely to resurface that content later.

\item[Coordination.]
Notice that in Figure \ref{fig:user-journey} the user seems to transition through many different kinds of experiences.
They see recommendations on the home page, but also see personalized results on the search page.
Some of their interactions may be with a personalized playlist; podcasts like \emph{Podcast X} may be embedded within certain mixed-audio playlists.
Distinct, focused, teams optimize each such experience for immediate outcomes, like whether a user streams after a search result.
Optimizing for the longer-term seems to require coordination among all such teams.
\end{description}

We have found that our structured modeling was essential to estimating the $Q$-function from historical (i.e. ``offline'') data in  a reasonably data-efficient manner, succeeding in the face of measurement challenges. 
Illustrative experiments suggest this reduced sample size requirements by as much as a factor of 120,000 compared to more generic approaches (Section \ref{sec:samplesize}). 

The issue is that user behavior on recommender systems is highly idiosyncratic, enough so that the impact of substantial, persistent changes to recommendations, as measured through large-scale A/B tests, is often indistinguishable from the natural ``noise'' in user behavior. A $Q$-function aims to measure a comparatively tiny impact---that of deviating from the incumbent policy for a single recommendation---on some long-term app-level reward metric (e.g. total listening minutes across months).  To draw an analogy, this is like aiming to measure the impact of a single unhealthy meal on overall lifetime health (e.g., quality-adjusted life years), toward the goal of improving a diet routine.  This poor signal-to-noise ratio is a feature of recommender systems, broadly, and made it especially difficult to apply more black-box RL approaches that have been successful in other domains, like arcade games or robotics. 

Our practical solutions instead leverage domain knowledge about the mechanism through which recommendations contribute to long-term user satisfaction (the formation of item-level habits). The nature of this solution may offer valuable insights about the kinds of approaches one might need to uncover (automatically or manually) in complex, real-world recommendation systems.

\subsection{Outline}
The remainder of this paper is organized as follows.
Section \ref{sec:literature} provides a review of related literature.
Section \ref{sec:formulation} presents a fairly generic RL model for optimizing recurring user interactions. 
Sections \ref{sec:structural}-\ref{subsec:challenges-revisited} are the core contributions of the paper.
First, Section \ref{sec:structural} introduces our domain-specific modeling of the $Q$-function, including our novel user-state representation and structural assumptions.
Section \ref{sec:prototypes} details our real-world implementations and their impact, presenting results from online and offline experiments.
Section \ref{subsec:challenges-revisited} revisits the challenges introduced in Section \ref{subsec:challenges}, explaining how our approach copes with them, and includes an empirical analysis of data efficiency.

\section{Related literature}\label{sec:literature}

\paragraph{Recommender systems.}
Recommender systems emerged alongside the World Wide Web in the 1990s, helping to filter and personalize the massive volume of content available.
Recommendation strategies are often divided in to three categories: content-based, collaborative filtering, and hybrid.
A content-based  approach relies on observable properties of  items (e.g., their text description) and recommends items to users that are similar to ones they engaged with previously \citep{mooney2000content}.
Collaborative filtering approaches instead use historical data on interactions between users and items to make recommendations \citep{su2009survey}.
Matrix factorization techniques \citep{koren2009matrix} are a popular collaborative filtering approach.
Hybrid approaches aim to mix both sources of information; see also
\cite{ansari2000internet} for a Bayesian preference model that allows many types of information to be leveraged.

Today, deep learning based systems are dominant in industrial recommender systems \citep{hidasi2016session, wang2021survey}.
In addition to their ability to fit complex nonlinear patterns, a critical advantage of such methods is their ability to incorporate and synthesize distinct sources of information.
See \citet{steck2021deep} for a critical evaluation of the impact of deep learning at Netflix.

A great deal of recent research aims to make recommender systems responsive to users' most recent interactions \citep{beutel2018latent, wang2021survey}.
Most such work still optimizes individual recommendations myopically, and is complementary to efforts to optimize sequences of recommendations for the long term.
Below, we describe several threads of research related to optimizing user interactions for the long-term.

\paragraph{Surrogate outcomes and proxy-metrics.}
Recognizing the drawbacks of optimizing recommendations only for immediate engagement, many have tried to develop proxy metrics that are better aligned with the long-term.
This approach is common in many fields---not just recommender systems.
For instance, when the delay involved in measuring long-term treatment effects is too severe, clinical trials often determine success based on surrogate endpoints \citep{prentice1989surrogate}.
A surrogate endpoint is typically a health indicator (e.g. blood pressure) that is expected to mediate a treatment's effect on the true outcome of interest (e.g. whether the patient suffers a heart attack.)

The use of surrogate or proxy metrics is common in recommender systems.
For instance, \cite{zhou2010solving}, \cite{ziegler2005improving} and \cite{anderson2020algorithmic} study the benefits of diversity in recommendations.
\cite{wang2022surrogate} recently studied five surrogate outcomes at a large online video streaming service.
These are mathematical functions, computed based on user interactions, which represent Diversity, Repeated Consumption, High-Quality Consumption,  Persistent Topics, and Page-Specific Revisits.
They improve some longer term outcomes by training the machine learning system to optimize for these metrics, rather than just immediate engagement.

Several papers have proposed multi-armed bandit models where surrogate outcomes encode actions' long-term impacts.
These include bandit models where poor recommendations cause attrition \citep{ben2022modeling, bastani2022learning} and bandit models where objectives incorporate diversity/boredom considerations \citep{xie2022multi, cao2020fatigue, ma2016user}.
\cite{wu2017returning} studies a variation on typical bandit model where actions impact whether a user will return to the system.
\cite{yang2020targeting} combine bandit-style exploration with long-term surrogate models to optimize promotional discounts at the Boston Globe, a prominent newspaper.

Surrogacy assumptions are similar to Markov assumptions in RL. One of our main assumptions in Section \ref{sec:structural} can therefore be viewed as surrogacy assumption.
Approaches that combine several surrogate outcomes into a scalar surrogate index \citep{athey2019surrogate} are intellectually similar to RL approaches that fit an approximate value function and then choose actions to increase value-to-go.
There tends to be a distinction in terms of granularity, however.
Most RL implementations use high dimensional state variables and fit value function approximations with neural networks, rather than handcraft a small number of surrogate outcomes.

\paragraph{RL for optimizing a recommendation systems.}
In recent years there have been many impressive demonstrations of algorithms that synthesize reinforcement learning principles with deep learning.
Recommender systems would seem to be a natural application of these techniques:
they generate massive volumes of data, interact sequentially with users, and deploy neural networks to optimize each  interaction.
Early proposals to view recommendation as an RL problem \citep[see, e.g.,][]{shani2005mdp} have been revisited with renewed interested in the past few years.
\cite{afsar2021reinforcement} provide a sweeping survey of over a hundred conference and workshop papers on this topic.
Unfortunately, numerous challenges can inhibit deployment reinforcement learning in the real world \citep{dulac2021challenges}.

We will focus our discussion in this section on a few high level themes in the literature, aiming to highlight the primary distinguishing features of our work. One segment of this literature uses RL to address single-period decision-making problems with combinatorial decision-spaces.
For instance, consider the problem of selecting a full page or `slate' of recommendations.
This can be broken down, via dynamic programming, into the problem of selecting the first item, then the second, and so on.
See \cite{zhao2018deep,xie2021hierarchical} or \cite{netflix2022budget} for recent works along these lines.
This thread of work is not so closely related to ours.
Another segment of work tries to use model-based RL \citep{hu2017playlist,chen2019generative, shang2019environment}.
Since user behavior is notoriously complex and idiosyncratic,  most papers instead try to directly estimate a few relevant quantities of interest from logged data, like value functions or policy gradients.
We focus most of our discussion on these works.

Several papers propose to optimize sequential user interactions by fitting---and maximizing---approximate $Q$-functions.
\cite{zou2019reinforcement} describe such an approach for the problem of optimizing a sequence of product recommendations appearing in a feed.
\cite{zheng2018drn} looks at optimizing a sequence of news article recommendations.
\cite{xin2020self} proposes a method that is similar to  an actor-critic algorithm, but with an augmented self-supervised loss.
\cite{chen2022off} describe a successful implementation at YouTube of an actor critic algorithm for candidate generation\footnote{As is common, their system relies on two separate pieces of logic: a candidate generation algorithm which filters a corpus with millions of items down to hundreds and a ranking algorithm which orders those on the page.}.
Actor-critic algorithms fit approximate $Q$-functions and use them to estimate directions for local policy improvements \citep{konda1999actor,sutton1999policy}.
\cite{chen2019top}  had previously applied the REINFORCE algorithm \citep{williams1992simple} to this problem.

Another example is provided by \cite{ie2019reinforcement}. Again, they test an RL based recommendation algorithm in live experiments at YouTube.
They focus on the problem of optimizing the total engagement time before a recommendation session ends.
To our understanding, a session begins when a user clicks on a video from the homescreen and ends when they exit the app or return to the homescreen.
In the interim, whenever a video is completed a full slate of other videos is recommended.
\cite{ie2019reinforcement} fits an approximate $Q$-function and then recommend slates that maximize it.
To make the problem tractable, they assume the $Q$-function decomposes additively across items and modify RL iterations to fit this assumption.
That work also contains many helpful discussions of engineering efforts around serving, logging etc. and develops a package for conducting simulation experiments \citep{ie2019recsim}.

Our paper has much in common with these works.
We are motivated by similar pitfalls of myopic recommendation strategies, we grapple with the problem of defining states and rewards, and we estimate $Q$-functions from offline data.
Like \cite{chen2019top,chen2022off}, our modeling decisions were heavily influenced by the need to work within a complex industrial scale system.

Our paper is distinct in using domain-specific modeling to enable optimizing for sustainable listening habits that take months to unfold.
Related papers appear to consider much shorter planning horizons.
\cite{ie2019reinforcement} optimize performance over a single viewing session.
\cite{zheng2018drn} and \cite{zou2019reinforcement} use a discount factors of 0.4 and 0.9, respectively, in problems where one recommendation is made per time period. Implicitly, they optimize performance over a short sequence of recommendations.
Optimizing instead over a long horizon greatly exacerbates the coordination and measurement challenges highlighted in Section \ref{subsec:challenges}.
As a result, we cannot directly apply unstructured actor-critic algorithms as in some of the aforementioned papers (see Section \ref{sec:samplesize}).
To overcome this challenge, we make structural assumptions motivated by domain-knowledge and design new features of user/item representations that are tailored to creating sustainable item-level habits.
The modeling in Section \ref{sec:structural} appears to be quite different from past work.

\cite{besbes2016optimization} study a non-myopic approach to news article recommendation.
They characterize content along two dimension: ``clickability, the likelihood to click to an article when it is recommended; and (2) engageability, the likelihood to click from an article when it hosts a recommendation.''
They formulate the problem of maximizing clicks across an entire visit to the system, and propose a one-step look-ahead policy that balances immediate click-through rate and engageability.
Their introduction of a new engageability metric is reminiscent of our introduction of item stickiness models.

\paragraph{Other MDP models of recurring customer interactions.}

Markov decision process (MDP) models of recurring customer interactions are almost as old as the formal study of MDPs themselves. \cite{howard2002comments} comments that his invention of policy iteration in 1960 was inspired by the real-life application of MDPs in optimizing catalog mailing policies at Sears, Roebuck and Company. See \cite{gonul1998optimal, simester2006dynamic} or \cite{gonul2006compute} for more recent published work along those lines.
More recently, but still over two decades ago, \cite{pfeifer2000modeling} advocated for the use Markov chain models of customer lifetime value, and MDP models of managing customer relationships.
\cite{rust2006marketing} calls optimizing dynamic marketing interventions for the increase in customer lifetime value the `Holy Grail' of customer relationship management.
The use of hidden Markov models of latent customers relationship states is also well established in the academic marketing literature \citep{netzer2008hidden, montoya2010dynamic, abhishek2012media, ascarza2013joint, liberali2022morphing}.

One major dimension along which we differ from this work is the granularity of personalization.
Most aforementioned papers deal with simple structured problems  \citep[e.g.][]{gonul1998optimal} or problems where user states can take on only a few possible values \citep[e.g.][]{netzer2008hidden} that make dynamic programming or hidden Markov modeling tractable and interpretable.
Successful recommendation systems need a very refined understanding of user tastes and item's appeal, and modern ones do this by training neural networks on billions of past user interactions.
Our work integrates the logic of MDPs into such a system, while retaining the richness and complexity of the real world problem.

\section{An RL Model of our objective: using historical data to  improve a component of a recommendation policy}
\label{sec:formulation}

\begin{figure}[t]
\centering
\begin{subfigure}{.4\textwidth}
    \centering
    \caption{Home banner}
    \vspace{3mm}
    \label{fig:screenshot-banner}
    \includegraphics[width=.7\linewidth]{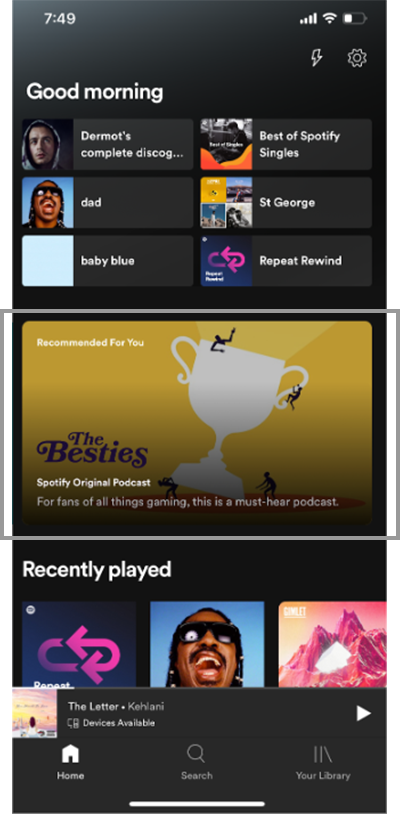}
\end{subfigure}
\begin{subfigure}{.4\textwidth}
    \centering
    \caption{Podcast discovery shelf}
    \vspace{3mm}
    \label{fig:screenshot-shelf}
    \includegraphics[width=.7\linewidth]{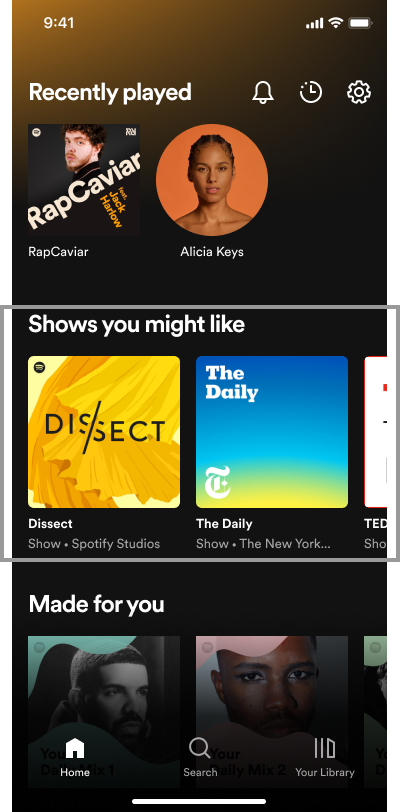}
\end{subfigure}
\caption{Screenshots of Spotify mobile application.
The banner component displays a single content item, whereas the podcast discovery shelf contains up to 20 cards that the user can scroll through.
}
\label{fig:screenshots}
\end{figure}

This section formalizes our main problem: using offline data to improve a component of the incumbent recommendation policy, aiming to enhance the total reward associated with users' recurring interactions with the system. We use the language of RL to model these interactions, culminating in the definition of an appropriate $Q$-function in \eqref{eq:partialQ}. This function quantifies how short-term deviations from the incumbent recommendation policy in a specific part of the app alter long-term performance. We aim to estimate this function using historical trajectories of user-states and recommendation actions, along with a reward function (which quantifies the desirability of outcomes from recommendations), and without estimating or specifying other components of the theoretical model.

Our practical implementations, derived through tailored, domain-specific modeling of these $Q$-functions, are presented in subsequent sections. Readers more interested in these  aspects may wish to focus on the latter part of this section, particularly the definition of the $Q$-function, while skimming other parts to grasp important notation.

\paragraph{A specific recommendation task to motivate abstract modeling.}
The solutions we derive have been implemented in several parts of the Spotify app, at scale. To explain our modeling choices, it is helpful to describe a specific recommendation task which has an especially clean structure.  This task is the basis for one of the A/B test described later in Section \ref{sec:prototypes}. 
Figure \ref{fig:screenshot-banner} shows a screenshot of the Spotify mobile app. At the very top are shortcut icons, making it easy for the user to return to previously listened to albums, playlists, or podcasts.
Just below that is a banner, which prominently recommends a particular podcast to the user.
The user may choose to scroll past the banner, encountering further personalized recommendations.
Through search results or playlists---including mixed-audio playlists which incorporate talk audio---they may encounter personalized recommendations of various items.
Often, when a user opens the app there will be no banner present. We will consider the problem of using historical (i.e. ``offline'') data to improve the policy that governs such banner recommendations---a select component of the broader recommendation policy in use at the system. 

\paragraph{Discrete time-period modeling of sequential of recommendations, engagement, and rewards.}
 We use the language of reinforcement learning to model a specific (randomly selected) user's recurring interactions with the app. In our model, the interactions occur across a sequence of discrete periods, denoted by $t$, which we think of as separate days. This timescale reflects our primary interest in modeling impact over a long time horizon---like multiple months---rather than creating a system that changes rapidly in response to user behavior throughout a given day.  We model retention as exogenous\footnote{It is possible to represent an endogenous decision to engage with no content, but we model de-activation as exogenous to limit the anticipated length of user journeys in the theory.}: the user becomes active in some period $T_0$ and, in each subsequent period, the user has a fixed probability $\gamma \in (0,1)$  of deactivating, independent of their app interactions.  We denote the de-activation period by $T_1$. The scalar $\gamma$ acts as an implicit discount factor. 

\begin{table}[t]
  \caption{
Symbols and notation introduced in Section~\ref{sec:formulation}. }
  \label{tab:formulsymb}
  \centering
  \begin{tabular}{lll}
  \toprule
  Symbol             & Domain                            & Description \\
  \midrule
  $T_0$              & $\Zbb$                            & Period at which user creates an account \\
  $T_1$              & $\Zbb$                            & Period at which user deactivates their account \\
  $\gamma$           & $(0, 1)$                          & Exogenous, fixed retention rate \\
  $X_t$              & $\Xbb$                            & Context of user at period $t$ \\
  $L$                & $\Nbb$                            & Number of recommendations per period \\
  $A_t$              & $\Abb^L$                          & Recommendations exposed to user at period $t$ \\
  $Y_t$              & $\Rbb_{+}^L$                      & Consumption of user at period $t$ \\
  $S_t$              & $\Sbb \subseteq \Xbb \times \Hbb$ & State of user at period $t$ \\
  $\pi(\cdot)$       & $\Sbb \to \Abb^L$                 & Recommendation policy \\
  $R(\cdot, \cdot)$  & $\Rbb_+^L \times \Abb^L \to \Rbb$ & Reward function, applied to $(Y_t, A_t)$ \\
  $\pi^0(\cdot)$     & $\Sbb \to \Abb^L$                 & Incumbent recommendation policy \\
  $\star$            & $[L]$                             & Recommendation position under consideration \\
  $\pi_\star(\cdot)$ & $\Sbb \to \Abb_\star$             & Policy component for position under consideration \\
  $Q_{\pi^0}(s,a)$ & $\Sbb \times \mathbb{A}_\star \to \mathbb{R}$             & Partial state-action value function given in \eqref{eq:partialQ}.  \\
  \bottomrule
\end{tabular}

\end{table}

At the start of each period (or ``day'') $t \in \{T_0, \ldots, T_1\}$, the system makes multiple recommendations ($L$ per day), denoted $A_t=(A_{t,1},\ldots, A_{t,L}) \in \Abb^L$, and these persist until the next period begins. It is helpful to think of $A_t$ as being arranged in a pre-specified vertical order, like a static one-dimensional version of the page in Figure \ref{fig:screenshots}.   Due to this rough analogy, we call $\ell \in [L]=\{1,\cdots, L\}$ a \emph{position}. 
We think of the number of positions $L$ as being a large number, with the user skipping over most recommendations without much thought. This abstraction reflects that users may encounter recommendations in varied ways throughout the day: they might see a banner recommendation, scroll through personalized lists, or encounter recommendations in search results or playlists. The same item may be recommended at multiple positions. 

The user engages with the recommendation in position $\ell$ for $Y_{t,\ell}\in \mathbb{R}_+$ seconds on the $t^{\rm th}$ day, with  $Y_{t,\ell}=0$ indicating the item is immediately skipped.
Let $Y_{t}=(Y_{t,1},\cdots, Y_{t,L})$.
We assume the user does not listen to items that are not recommended\footnote{This assumption is more innocuous than it may appear at first. In our model, a large number of recommendations per day (i.e. $L-1$) are governed by an incumbent policy and we are constrained to only modify the policy component that governs a single recommendation. 
Search results can be thought of an extension where a user provides a query which seeds some components of these $L-1$ recommendations.
Adding user queries to the model seems to complicate the presentation without yielding new insights.}.

The holistic reward function $R: \mathbb{R}_+^L \times \Abb^L \to \mathbb{R}$ associates a user's daily interactions with a measure of success. As simple examples, $R(Y_t, A_t) =\sum_{\ell=1}^{L} Y_{t,\ell}$ rewards total engagement time throughout the day whereas $R(Y_t, A_t) =\ind\left(\sum_{\ell=1}^{L} Y_{t,\ell}>0\right)$ assesses whether the user engaged with some recommended content that day. 

\paragraph{Theoretical generative model of user behavior.} We now state an abstract model of user behavior. The purpose is to ensure our mathematical statements are completely rigorous (e.g., expectation operators are well-defined). Beyond this, it does not play a significant role in our practical solutions.

First, we assume user transitions through a sequence of  contexts $(X_t : t\in \{T_0, \cdots, T_1\} )$  representing factors like a user's age, device, or the day of the week. Assume this follows an exogenous Markov chain, and that, conditioned on $X_t$, $X_{t+1}$ is independent of $A_t$ and $Y_t$ as well as all prior contexts, recommendations, and streaming decisions.

The user's response to recommendations $A_t$ is influenced by their history prior to day $t$, $H_{t} = \{ (X_{\tau}, A_{\tau}, Y_{\tau}) :  \tau \in \{T_0, \cdots, t-1\} \} \in \mathbb{H}$, the context $X_t$, a latent and unobservable type of the user $\omega \in \Omega$, and an unobservable shock $\xi_t  \in \Xi$ drawn independently and identically across periods and independently from all else.
Formally, there is some fixed function $\psi$ that determines the user response  as
\[
Y_t=\psi(H_t, A_{t}, X_t, \omega, \xi_{t}).
\]
The variable $\xi_{t}$ is meant to capture idiosyncratic randomness in user behavior.
The latent type $\omega$ is meant to model persistent user preferences that are unknown to the recommender but might be partially inferred from their behavior.
Since the user is randomly drawn from a population, $\omega$ is a random variable.
The dependence of future behavior on a user's history means that recommendations in one period can influence the efficacy of recommendations in the future.

\paragraph{Recommendation policies and  user lifetime reward.} 
A policy is a (possibly randomized) rule that determines the $L$ recommendations as function of the state of the user $S_t$.  For now, $S_t = (X_t, H_t)$
denotes the exhaustive list of all information the system has about the user's previous interactions  and context\footnote{Viewing the state variable as reflecting the state of the decision-maker's knowledge is classical in treatments of partially observable Markov Decisions Processes \cite[][Chapter 4]{bertsekas2012dynamic}; See \citet[][Chapter 17.3]{sutton2018reinforcement} or \citet{lu2021reinforcement} for a discussion of encoding the history as a state variable. The recommender's action $A_t$ is necessarily a function of the state variable, and perhaps some exogenous randomness, since it has no other information on which to base its decisions.
This state variable trivially obeys the Markov property $\Prob(S_{t+1} \in \cdot \mid S_t, A_t) =  \Prob(S_{t+1} \in \cdot \mid S_1, A_1, \cdots, S_{t}, A_t)$, since $S_t$ encodes all information in $(A_{1},S_1,\cdots A_{t-1}, S_{t-1})$. }. (Tailored, parsimonious, representations are developed in Section \ref{subsec:pragmatic-user-state}.)  If $t\notin \{T_0,\cdots, T_1\}$, indicating the user is inactive, we write $S_t=\varnothing$. Let $\Sbb$ denote the set of all possible states.

A policy $\pi$ has value function $V_{\pi}:\Sbb \cup \{\varnothing\} \to \mathbb{R}$ defined as
\begin{equation}
\label{eq:value-function}
    V_{\pi}(s) = \mathbb{E}_{\pi}\left[ \sum_{\tau=t}^{T_1}  R(Y_\tau, A_\tau)  \mid  S_{t}=s  \right]  \qquad \text{for all } s\in \Sbb,
\end{equation}
with $V_{\pi}(\varnothing)=0$. The value function can also be written in terms of the sum of discounted rewards over an infinite horizon (Remark \ref{rem:discounting}).  The subscript $\pi$ indicates that actions are selected according to $\pi$. A value function defines a partial order over policies, with one policy outperforming another if it offers higher value-to-go for every possible state. It is sometimes helpful to also consider a scalar objective function,
\begin{equation}\label{eq:avg-value-function}
J(\pi) = \E\left[ V_{\pi}(S_{T_0}) \right],
\end{equation}
which considers the average lifetime reward for a new user.

\paragraph{The incumbent policy and logged data.}
We isolate a specific policy, $\pi^0$, which we call the \emph{incumbent policy}. This represents the status-quo policy in use on the recommender system. It plays two roles. First, it is the natural benchmark against which the performance of other policies is judged. In A/B tests, it governs recommendations for users in the ``control group'. 

Second, it is the policy under which historical data was collected, sometimes called the ``behavioral policy'', or a ``logging policy'' in offline reinforcement learning. In particular, we assume access to a batch of trajectories of user interactions
\begin{equation}\label{eq:user-trajectories}
\mathcal{D} = \{ (S_{t}^u, A_t^{u}, Y_t^{u})  : t\in \{T_0^u, \cdots, T_1^u\} \}_{u \in U },
\end{equation}
ranging over users in a finite set $U$ whose interactions were with  the incumbent policy $\pi^0$.
Our ultimate system will need to involve quantities that can be estimated using such data.

We assume that the incumbent recommendation policy has a nonzero chance of selecting each action in each state. One can interpret this as a source of exploration and unbiased training data.
Much of our motivation for making Assumption \ref{ass:overlap} is expository. See Remark \ref{rem:causal} below.
\begin{assumption}[Randomness in action selection]\label{ass:overlap} The incumbent policy  is a function $\pi^0: \Sbb \times \Xi^{L} \to \Abb^L$ which associates a state in $\Sbb$ and an exogenous noise vector in $\Xi^L$ with a sequence of items.  Assume there is an i.i.d. collection of random variables $\epsilon = (\epsilon_{t, \ell} :  t\in \mathbb{Z}, \ell\in [L] )$ taking values in $\Xi$, such that $A_{t,\ell}=\pi^0_{\ell}(S_t, \epsilon_{t,\ell})$. Moreover,   For each $s\in \Sbb$, $\Prob( \pi_{\star}^{0}(s, \epsilon_{t,\star}) =a_{\star} )>0$ for each $a_{\star} \in \Abb_{\star}$.
\end{assumption}

\paragraph{Improving a component of the policy.}
We focus on a specific component of this incumbent policy, denoted $\pi^{0}_{\star}$, which might represent, for example, the logic governing podcast recommendations on the app's home banner. In particular, $\pi^0_{\star}$ selects among a restricted set of items $\mathbb{A}_{\star}$, a specific item $A_{t,\star}$ to display at position $\star$, on the basis of the current user state $S_t$ (allowing for randomization). 

Our goal, roughly speaking, is to increase the lifetime average reward in \eqref{eq:avg-value-function} by adjusting the recommendation policy at position $\star$.
To describe this concretely, we introduce the partial state-action value function $Q_{\pi^0}: \Sbb \times \Abb_{\star} \to \mathbb{R}$ defined by
\begin{equation}\label{eq:partialQ}
    Q_{\pi^0}(s,a_\star) = \mathbb{E}_{\pi^0}\left[  \sum_{t=0}^{T_1}  R(Y_t, A_t)  \mid S_0 = s, A_{0,\star} =a_{\star} \right].
\end{equation}
This measures the expected remaining lifetime reward when item $a_\star$ is recommended at position $\star$ to a user whose current state is $s$ and the incumbent policy $\pi^0$ is used to determine recommendations at other positions and in future periods. In other-words, this models the long-term, holistic, impact of short deviations from the incumbent policy.

While \eqref{eq:partialQ} appears to consider a single period decision-problem, it can be used to derive policies with superior performance in long-horizon problems.
In particular, the policy iteration update $\pi^+$, defined by $\pi_\star^+(s) \in \argmax_{a_\star \in A_\star} Q_{\pi^0}(s,a_\star)$ and $\pi_\ell^+(s) = \pi^0_\ell(s)$ for $\ell \neq \star$, satisfies $V_{\pi^+}(s) \geq V_{\pi^0}(s)$ for every $s\in \Sbb$   \citep{bertsekas2012dynamic}.
This implies $J( [\pi^+_{\star}, \pi_{\smallsetminus \star}^0]) \geq J( [\pi^0_{\star}, \pi_{\smallsetminus \star}^0])$.  Appendix \ref{sec:policy_improvement_theory} reviews theory that lets us interpret this policy improvement process as a \emph{coordinate ascent} method based on the gradient  of $J$, usually called  the ``policy gradient''. 

We aim to approximate $\pi^+_\star$ using logged data $\Dscr$.  Since the $Q$-function in \eqref{eq:partialQ} is defined by an expectation under the incumbent policy, there is hope of approximating it using the logged data $\Dscr$. In theory, this policy improvement process would eventually produce the true optimal policy if were conducted repeatedly, though separately, on each policy component. Here we are instead satisfied with a single policy update, producing a (substantial) improvement to the current status-quo.

\begin{remark}[Implicit discounting]\label{rem:discounting}
    The formulation defines an interaction protocol that could be used to generate an infinite trajectory $(S_t, A_t, Y_t : t\in \{T_0, T_0+1, T_0+2, \cdots  \})$. By integrating over the random, geometric lifetime $T_1-T_0$, which is independent of this trajectory, one can show that
    \[
    V_{\pi}(s) =  \mathbb{E}_{\pi}\left[ \sum_{\tau=t}^{\infty}  \gamma^{\tau-t} R(Y_\tau, A_\tau)  \mid  S_{t}=s  \right]  \qquad \text{for all } s\in \Sbb.
    \]
    By creating a model in which users deactivate with probability $1-\gamma$ independently in each period, we effectively formulate a problem with a discounted infinite horizon objective.
\end{remark}

\begin{remark}[Causal interpretation of conditional expectations]\label{rem:causal}
    Consider the difference in a user's chance of engaging with the recommendation at position $\star$ under two possible recommendations:
    \[
    \mathbb{P}_{\pi^0}(Y_{t,\star} >0 \mid A_{t,\star} =a_\star, S_t) - \mathbb{P}_{\pi^0}(Y_{t,\star}>0 \mid A_{t,\star}=a'_\star, S_t).
    \]
    These conditional probabilities are well defined due to Assumption \ref{ass:overlap}.
    Because, conditioned on the state, the recommendation $A_{t,\star}$ at position $\star$ is a function of exogenous randomness specific to that decision, the difference above can be thought of as the causal impact of choosing to recommend item $a_\star$ rather than $a'_\star$.
    In particular, in the notation of \cite{pearl2009causal},
    \[\mathbb{P}_{\pi^0}(Y_{t,\star} >0 \mid A_{t,\star} =a_\star, S_t) = \mathbb{P}_{\pi^0}(Y_{t,\star} >0 \mid do\left(A_{t,\star} =a_\star\right), S_t).
    \]
    For this reason, we are able to avoid specialized causal inference notation like that of \cite{pearl2009causal} or \cite{rubin1974estimating} throughout the paper.
    The focus on the event $\{Y_{t,\star} >0\}$ was for concreteness only.
    An analogous statement holds for any event depending only on the user's future recommendations and behavior.
\end{remark}

\section{Domain-specific modeling of the $Q$-function}
\label{sec:structural}

While our discussion of offline $Q$-estimation in the previous section was largely generic, applicable to a wide range of recommendation systems, this section marks a significant shift towards domain-specific modeling. Here, we propose and leverage a hypothesis about the nature of audio listening habits to develop a structural approach for modeling the $Q$-function in audio streaming recommendation systems.

As described earlier in  the introduction, our hypothesis is that a critical way in which recommendations contribute to users' long-term satisfaction is through aiding in the formation of item-level listening habits. Rather than aim to make recommendations that ``increase lifetime reward'' in a black-box way, we specialize our modeling to this specific mechanism. For instance, we posit that a key recommendation might help a user discover a podcast series they return to regularly, or a playlist that becomes part of their daily routine. This notion of `stickiness' – the tendency of  users to form lasting attachments to specific content – forms the cornerstone of our modeling approach.  
To tailor our $Q$-function modeling to this hypothesis, our approach consists of two key components:
\begin{enumerate}
    \item A practical representation of user-state that captures both overall listening preferences and specific content relationships (Section \ref{subsec:pragmatic-user-state}).
    \item Structural assumptions that simplify the $Q$-function, making it more amenable to real-world implementation while focusing on the formation of item-level listening habits (Section \ref{subsec:structured-Q}).
\end{enumerate}
This hypothesis-driven, structural modeling approach offers several advantages over generic or black-box methods. By focusing on a specific mechanism (the formation of item-level habits), we can more precisely measure and optimize for long-term user satisfaction (as measured by ``lifetime reward''). Moreover, the approach improves interpretability and aligns well with organizational structures in large-scale recommendation systems.

Section \ref{subsec:challenges-revisited} touches on how this solution overcomes attribution, coordination, and measurement challenges, including a simple demonstration of how our structural modeling enhances data efficiency.

\subsection{User-state representation}\label{subsec:pragmatic-user-state}
Building on our intuition about item-level listening habits, we now describe a practical representation of user state that captures essential aspects of user behavior and content relationships. Our representation consists of three key components, denoted by $\{u_t, Z_t, X_t\}$:
\begin{enumerate}
    \item \emph{Slow-moving vector representations of user's ``audio tastes''} $u_t$: learned, fixed-length, vectors that we think of as encoding the recommended system's understanding of the user's tastes. For instance, this vector may reflect a user's affinity for ``true crime'' podcasts.  Such representations are core to many recommender systems; we provide further discussion below and in Appendix \ref{sec:background}. 
    \item \emph{Fast-moving context $X_t$:} These features capture immediate contextual factors that might influence a user's short-term behavior, such as time of day, day of week, or recent app activity. These are used in pre-existing systems that predict user's short-term behavior in response to a recommendation.  
    \item \emph{Content-relationship state $Z_t=(Z_{t,a})_{a\in \mathbb{A}}$}: This component captures the user's history with each piece of content. For example, it might encode how recently and how frequently a user has listened to a particular podcast series. This component is a novel feature of our modeling approach.  
\end{enumerate}
Together, these components provide a comprehensive view of the user's state.
The slow-moving taste vectors capture overall preferences, the fast-moving context allows for short-term adaptability, and the content-relationship state enables our model to track and forecast evolving item-level habits.
The content-relationship state is particularly crucial to our approach.
Unlike the user vector $u_t$, which captures overall taste, the content-relationship state $Z_{t+1,a}$ with item $a$ is highly sensitive to whether the user engages with item $a$ today.
This allows us to model the long-term benefit of recommendation actions by modeling user-state transitions. 

\begin{table}[t]
  \caption{
Symbols and notation introduced in Section~\ref{sec:structural}. }
  \label{tab:structsymb}
  \centering
  \begin{tabular}{lll}
  \toprule
  Symbol             & Domain                               & Description \\
  \midrule
  $u_t$              & $\Rbb^*$                             & Slow-moving vector representing user's audio tastes \\
  $X_t$              & $\Xbb$                               & Fast-moving context of user \\
  $Z_t$              & $\Rbb^{\lvert \Abb \rvert \times k}$ & Content-relationship state \\
  $C_{t,a}$          & $\Rbb_+$                             & User's consumption of item $a$ at time $t$ \\
  $f(\cdot, \cdot)$  & $\Rbb^k \times \Rbb_+ \to \Rbb^k$    & update function, computes $Z_{t+1,a}$ from $(Z_{t,a}, C_{t,a})$ \\
  $r(\cdot)$         & $\Rbb_+ \to \Rbb$                    & Item-level reward component, applied to $C_{t,a}$ \\
  \bottomrule
\end{tabular}

\end{table}

Given their central role in our approach, we now provide more detailed explanations of the taste representations and content-relationship states.

\paragraph{Learned embeddings of users and items in ``taste'' space.}
We now explore in more detail the slow-moving vector representations of users' audio tastes mentioned earlier. Vector embeddings of users and items are central to many recommendation systems. These representations are typically trained by solving a surrogate classification problem, where the model learns to predict user interactions with items based on their vector representations. See Appendix \ref{sec:background} for an overview of influential work by \cite{covington2016deep}, which illustrates this approach.

We continue to denote an item (such as \emph{Podcast X}) by $a$ and denote the corresponding item vector by $\nu_a \in \mathbb{R}^d$. We denote the corresponding user vector at period $t$ by $u_t \in \mathbb{R}^d$. Both user and item vectors reside in the same $d$-dimensional space, allowing for direct comparisons of users and items.
We think of $u_t$ as encoding the recommender system's understanding of the user's tastes, and the dot product 
\begin{align}
  \label{eq:affinity}
  \mathrm{affinity}(u,a) = u_t^\top \nu_a
\end{align}
as encoding the propensity of a user with given tastes to have a short interaction with the item. This vector-based representation allows the recommendation system to efficiently compute and compare user-item affinities, facilitating the selection of items that align with a user's tastes.

\paragraph{Content-relationship states.}

We represent a user's content-relationship state as $Z_t \in \mathbb{R}^{|\Abb|\times k}$, consisting of a $k$-dimensional embedded representation of the user's consumption history for each piece of content. For this to function as a true state variable, we assume it can be updated incrementally:
\begin{equation}
    \label{eq:content-state}
    Z_{t+1,a} = f( Z_{t,a}, C_{t,a} ).
\end{equation}
Here, $C_{t,a}$ denotes the user's consumption of item $a$ at period $t$, defined as $C_{t,a}=\sum_{\ell=1}^{L} Y_{t,\ell} \ind(A_{t,\ell}=a)$, where $Y_{t,\ell}$ represents the engagement level with recommendation $\ell$ (e.g., listening time) and $\ind(A_{t,\ell}=a)$ indicates whether item $a$ was recommended in position $\ell$.

The update function $f$ could potentially be learned by training a recurrent neural network on a surrogate prediction task. However, our prototype implementations have used a simpler handcrafted encoding, described in Example \ref{ex:ema-content-state}. Either way, $f$ is treated as a known function. 

We initialize $Z_{T_0,a} =0 \in \mathbb{R}^k$, and assume that $f$ maps zero inputs to zero outputs, representing no prior engagement with the content. In practice, given the enormous content libraries in most large-scale recommendation systems, we expect a user's content state to be extremely sparse. This sparsity can be leveraged for efficient computation and storage.

\begin{example}[Exponential-moving-average relationship states]\label{ex:ema-content-state}
Our prototype implementations use a simple handcrafted encoding.
We update relationship sates as
\begin{equation}\label{eq:content-state-ema}
Z_{t+1,a} = \alpha \circ Z_{t,a} + (1-\alpha)\circ \ind\left(C_{t,a} >0\right)
\end{equation}
where $\alpha \in [0,1]^k$ is a vector of forgetting factors and $\circ$ denotes componentwise multiplication.
Each element of $Z_{t,a}$ is an exponential moving average of listening indicators. When a varied range of forgetting factors are used, the whole vector encodes the user's level of engagement with the item and how this has changed across time.
\end{example}

\subsection{Additively separable reward functions}

Toward decomposing the complex task of long-term optimization into item-specific reasoning, we consider reward functions which separate additively across items. In particular, the reward functions  we consider are based on a user's engagement or \emph{consumption} in period $t$, defined earlier as $C_{t,a}=\sum_{\ell=1}^{L} Y_{t,\ell} \ind(A_{t,\ell}=a)$.  Our methodology requires that immediate rewards are \emph{additively separable} functions of consumption, as $R(A_t,Y_t)=\sum_{a\in \mathbb{A}} r(C_{t,a})$, so that lifetime rewards decompose as
\begin{equation}\label{eq:separable-rewards}
    \underbrace{\sum_{t=T_0}^{T_1} R(A_t, Y_t)}_{\text{Holistic reward}}  = \sum_{a\in \mathbb{A}} \underbrace{ \sum_{t=T_0}^{T_1} r\left(C_{t,a} \right) }_{\text{Item-specific reward}}.
\end{equation}
With the choice $r(c)=c$, lifetime rewards track total engagement time. A strictly concave content-level reward $r(\cdot)$ prioritizes engagement with many pieces of content over extreme engagement with a single item. Our specific implementations correspond to the reward
\[
r(c)=\ind(c>0),
\]
in which case lifetime rewards value item-relationships which span across many days (i.e. periods $t$).

\subsection{A structured $Q$-function}\label{subsec:structured-Q}
In this section, we present a specialized approach to $Q$-function modeling, tailored to the hypothesis that recommendations critically contribute to users' long-term satisfaction through their impact on item-level listening habits. This tailored approach results in a specialized formula for the $Q$-function, which we derive under various structural assumptions. These assumptions will be detailed in Subsection \ref{subsec:assumptions_and_theorem}. 

To quantify these habits, we introduce the concept of `stickiness models' in the Subsection \ref{subsec:stickiness}. These models forecast the strength of a user's listening habit with an item, based on the system's understanding of their tastes and their history with the item.

Following this, Subsection \ref{subsec:counterfactuals} demonstrates how we assess the long-term value of recommendations. We do this by examining changes in projected stickiness, effectively measuring how a recommendation today might help a user form a lasting listening habit.

\subsubsection{Stickiness models.}\label{subsec:stickiness}
A key ingredient of our modeling approach is the concept of `stickiness models'. These models forecast the strength and durability of a user's engagement with a specific item over time. Crucially, these predictions are both personalized to individual users and tailored to specific items.

Formally, we define an item-level value function, or ``stickiness-model'', as
\begin{equation}\label{eq:stickiness_definition}
V_{\pi^0}^{(a)}\left(  z ; u \right) =\mathbb{E}_{\pi^0}\left[ \sum_{\tau=t+1}^{T_0} \underbrace{\ind\left(C_{\tau,a}>0\right)}_{=r(C_{\tau,a})}
  \mid Z_{t+1,a}=z \, ,\, u_t=u, S_{t+1} \neq \varnothing  \right].
\end{equation}
Recall that $C_{t,a}$ denotes the user's consumption of item $a$ at period $t$. This model calculates the expected number of future periods in which the user will engage with item $a$, given their current content-relationship state $z$ with the item and their overall taste vector $u$. Notice that the stickiness model in \eqref{eq:discovery_stickiness_prototype} is an idealized quantity, which we will eventually need to approximate via regression. 

The stickiness models defined above are personalized based on just two key factors: the vector representation of the user's taste and the user's content-relationship state with the item. While other user features could potentially be incorporated, these core components have proven sufficient in our prototype implementations.

\begin{remark}[Comment on stickiness estimation]
  Anecdotally, a myriad of factors determine whether an item is ``sticky'': stickiness could, for instance, be driven by intense fandom for a unique creator or the unique utility of a short news podcast on a specialized topic. When eventually estimating stickiness models from data, as described in Section \ref{sec:prototypes} and especially Subsection \ref{subsubsec:discovery_stickiness},  we will let the data speak for itself.  We gather historical user listening journeys with each item over a long time horizon, and train models that predict the strength of a user's lasting relationship with the item. These models are personalized based on a specific user's features, while also capturing item-specific patterns learned from the collective behavior of many users. This approach allows us to infer which items are likely to be sticky for particular users.   
\end{remark}

\subsubsection{Counterfactual valuation logic}\label{subsec:counterfactuals}

 We present a model for the $Q$-function that is used in our prototypes. At a high level, it assigns large value to recommendation actions that contribute positively to a listening habit, as measured by changes in item-level stickiness predictions. For now, we emphasize intuition for the formula;  Section \ref{subsec:assumptions_and_theorem} introduces formal assumptions under which it is truly ``correct'', rather than a pragmatic simplification of reality. 
 
\begin{cor}\label{cor:Q-function-prototype}  Under reward functions satisfying \eqref{eq:separable-rewards}, and  Assumptions \ref{assumption:direct-short-term-impact}-\ref{assumption:bonus2}, if $S_t \neq \varnothing$ then for each $a\in \Abb_\star$,
\begin{equation}
\begin{aligned}
\label{eq:resurfacing-qval}
Q_{\pi_0}(S_t, a) - b_{\pi^0}(S_t) =
& \underbrace{\Big[ \mathbb{P}_{\pi^0}(C_{t,a}> 0 | u_t, X_t, Z_{t,a}, A_{t,\star}=a) -   \mathbb{P}_{\pi^0}(C_{t,a}> 0 | u_t, X_t, Z_{t,a}, A_{t,\star}\neq a)\Big]}_{\text{impact of recommendation on consumption probability}}  \\
& \qquad\qquad \times \underbrace{\left[ \left(1 + \gamma V^{(a)}(u_t, Z_{t+1,a}^{+}) \right) - \left(0 + \gamma V^{(a)}(u_t, Z_{t+1,a}^{-}) \right)  \right]}_{\text{impact of consumption event on long-term engagement}}
\end{aligned}
\end{equation}
where $Z_{t+1,a}^{+} = f(Z_{t,a},1)$ and $Z_{t+1,a}^{-} = f(Z_{t,a},0)$
denote successor content-relationship-states that follow a listen and no-listen, respectively and $b_{\pi^0}(S_t)$ does not depend on the recommendation decision $a$; it is defined in Theorem \ref{thm:main}.
\end{cor}
We say the user ``listens'' to item $a$  if $C_{t,a}> 0$, whereas ``no-listen'' occurs that period if $C_{t,a}=0$.
The term $\mathbb{P}_{\pi^0}(C_{t,a}> 0 | u_t, X_t, Z_{t,a}, A_{t,\star}=a)$ denotes the probability that the user listens to item $a$ today if it is recommended at position $\star$. The probability conditioned on $A_{t,\star}\neq a$, corresponding to organic listens, is defined analogously.

The overall formula values a recommendation based on two factors:
\begin{enuminline}
\item its influence on whether a listen of that item occurs today, and
\item whether listen and no-listen events result in substantially different expected return days to the recommended item, as assessed through the stickiness model.
\end{enuminline}
One can estimate this $Q$-function by separately estimating models of short-term listening behavior and the long-term, item-level, stickiness models.

We have now derived a distinctive model for  the $Q$-function. Most notably, our modeling isolates the role of recommendations in  fostering item-level listening habits. Rather than directly making predictions about complex app-level interactions that are only loosely connected to the recommendation, all components of the $Q$-function formula focus on future engagement with the recommended item itself. Another notable feature  of the model is that it allows us to separate short-term and long-term effects, which can be estimated using different data sources and time horizons.

\subsection{Formal assumptions and derivation}\label{subsec:assumptions_and_theorem}

This section describes the assumptions and derivations underlying the formula in Corollary \ref{cor:Q-function-prototype}. We state assumptions under which the $Q$-function (effectively) decomposes across items and periods, a result that generalizes Corollary \ref{cor:Q-function-prototype}. Then, in the next subsection, we state some extra modeling choices that yield the specific form of Corollary \ref{cor:Q-function-prototype}.  

\subsubsection{Structural assumptions.}

\paragraph{Restricting to direct effects of a recommendation.}
It is natural to expect that recommending an item increases the chance the user listens to \emph{that} item.
We refer to this as the \emph{direct} impact of a recommendation.
As a simplifying assumption, we assume that recommendations have no indirect effects.
In words, the next assumption states that, given what is known about a user (i.e., $S_{t}$), and given that item $a_\star$ is not recommended, discounted future reward associated with consumption of item $a_{\star}$ does not depend on which other item ($A_{t,\star})$ is recommended. 
See Figure \ref{fig:causalgraphfull} for a graphical representation.
\begin{assumption}[No indirect impact of recommendation]
\label{assumption:direct-short-term-impact}
Under $\pi^0$, for every $a_{\star} \in \Abb_{\star}$,
\[
\sum_{\tau=t}^{T_1} r\left(C_{\tau,a_{\star}}\right)
    \perp A_{t, \star} \mid S_{t} \, , \, \ind\{A_{t,\star} = a_{\star}\}.
\]
\end{assumption}

Recommendations likely do have indirect effects.
There could be substitution effects, where a successful recommendation supplants other consumption, or spillover benefits, where a successful recommendation makes a user more likely to engage with others, or possible annoyance, where a poor recommendation causes a user to close the app.
To build a simple, pragmatic, system, we do not model complex indirect effects of recommendations. 

It is an interesting, though challenging, problem, to  causally assess the extent of substitution effects using available randomized data. A/B test results provide some assurance that engagement created via our recommendations cannot does not purely cannibalize other listening, since our policy substantially improves metrics like overall podcast listening (Section \ref{sec:discovery-home}).

\begin{figure}[t]
    \caption{
        Graphical model encoding Assumptions \ref{assumption:direct-short-term-impact} and \ref{assumption:markov-in-time}.}
    \label{fig:causalgraphfull}
    \centering
    \setlength{\fboxsep}{1pt}%
    \setlength{\fboxrule}{1pt}%
    \fbox{\includegraphics[width=.75\textwidth]{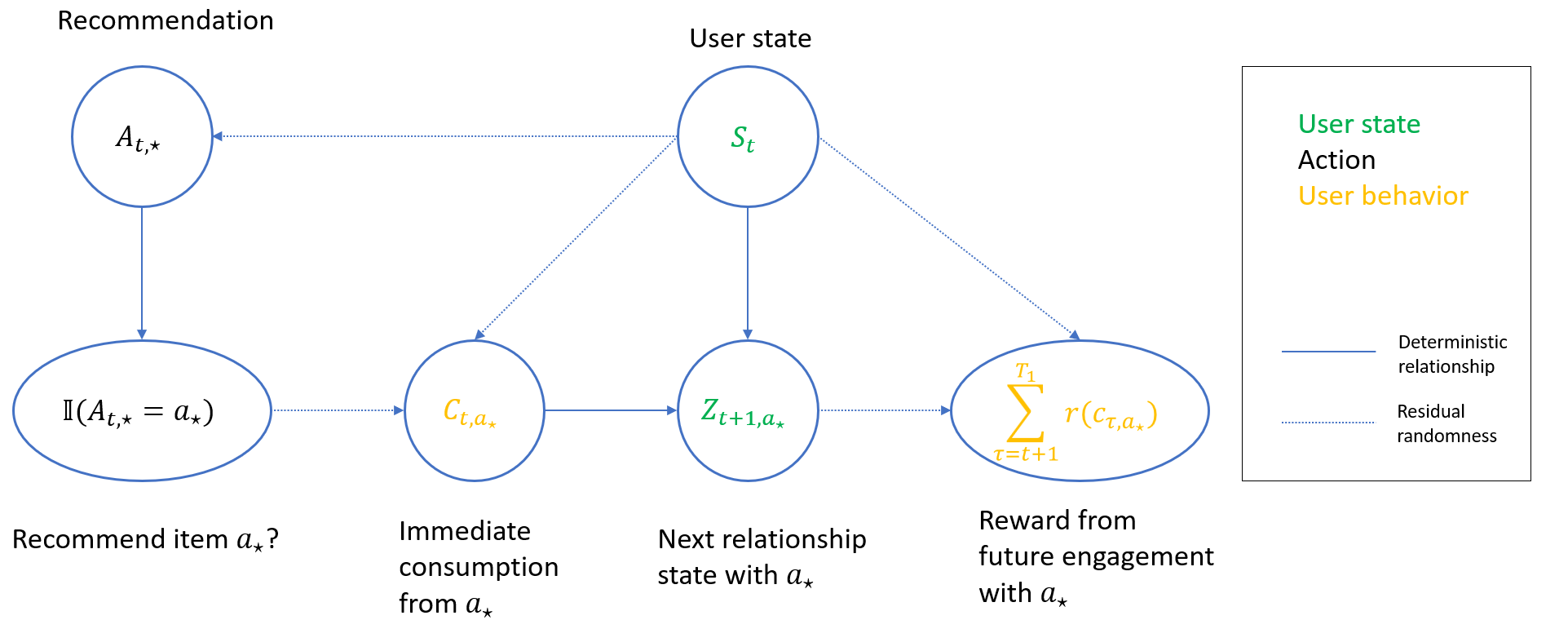}}
\end{figure}

\paragraph{Surrogacy.}

Our next assumption is that the impact of deciding to recommend an item on rewards associated with future engagement with that item is mediated through the recommendation's impact on the user's next engagement-state with the item. Figure \ref{fig:causalgraphfull} displays the conditional independence structures stated in Assumptions \ref{assumption:direct-short-term-impact} and \ref{assumption:markov-in-time}.
Conditioned on the user's state, the (randomized) decision of whether to recommend item $a_\star$ influences immediate consumption of item $a_\star$, which then influences the next relationship state with item $a_{\star}$, which may alter projected future engagement with that item.
Notice that this assumption conditions on the full state in the current period (i.e. period $t$) but the content-state in the \emph{next} period ($t+1$). 
\begin{assumption}[Surrogacy]
    \label{assumption:markov-in-time}
    Under $\pi^0$, for any item $a_*\in \Abb_\star$ and period $t$,
    \[
    \sum_{\tau=t+1}^{T_1}  r(C_{\tau,a_{\star}})
    \perp \ind\{A_{t,\star} = a_\star\} \mid S_{t}, Z_{t+1, a_\star}.
    \]
\end{assumption}
It is worth emphasizing that future rewards still depend on the item $a_\star$ itself. A critical part of our prototypes is in understanding which items users are likely to form recurring habits with.

That actions can be thought of as influencing the next state of the system is always true if the state variable encodes sufficient information about the history of observations (e.g., consider the exhaustive state in Section \ref{sec:formulation}). Such Markov assumptions become stringent---and powerful---when the state variable selectively forgets aspects of the observation history. This ``forgetting'' enhances statistical efficiency through a data-pooling effect \citep{cheikhi2023statistical}. In our model, content-relationship states capture a user's level of engagement with an item while disregarding the specific path that led to that engagement. Essentially, we assume that the discovery of a new item via a recommendation on the home banner (Figure \ref{fig:screenshots}) results in the same future engagement as if the user had found the item through active search. Or A/B tests provide evidence supporting this modeling assumption, as discussed in Section \ref{sec:discovery-tars}.

\subsubsection{General result: value function decomposition.}
 We state a result that generalizes Corollary \ref{cor:Q-function-prototype}. 
Note that the conditional expectation in the formula for $Q_{\pi^0}$ integrates only over the consumption of item $a$, $C_{t,a}$, since the content-relationship state $Z_{t,a}$ is a known given the full state $S_t$.
\begin{theorem}\label{thm:main} For additively separable reward functions satisfying \eqref{eq:separable-rewards}, and Assumptions \ \ref{assumption:direct-short-term-impact}  and \ref{assumption:markov-in-time}, if $S_t \neq \varnothing$ then for each $a\in \Abb_\star$,
        \begin{align*}
        Q_{\pi^0}(S_t, a) = b_{\pi^0}(S_t)  &+  \mathbb{E}_{\pi^0}\left[ r(C_{t,a})    +  \gamma V^{(a)}_{\pi^0}\left( f\left(Z_{t,a}, C_{t,a} \right)   \,,\, S_t\right)   \mid A_{t,\star} =  a \,,\, S_{t} \right]\\
        &- \mathbb{E}_{\pi^0}\left[ r(C_{t,a})    +  \gamma V^{(a)}_{\pi^0}\left( f\left(Z_{t,a}, C_{t,a} \right)   \,,\, S_t\right)   \mid A_{t,\star} \neq  a \,,\, S_{t} \right],
    \end{align*}
    where $Z_{t+1,a}=f\left(Z_{t,a}, C_{t,a} \right)$ is the next content-relationship state with item $a$,
    \begin{align}\label{eq:item-level-value-func}
V^{(a)}_{\pi^0}(z_a;  s) = \mathbb{E}_{\pi}\left[ \sum_{\tau=t+1}^{T_1} r\left(C_{\tau,a}\right)
  \mid Z_{t+1,a}=z_a \, ,\, S_t=s, S_{t+1} \neq \varnothing  \right]
\end{align}
is an item-level value function (or stickiness model) and the baseline value
    \[
    b_{\pi^0}(S_t) = \sum_{a\in \Abb} \mathbb{E}_{\pi^0}\left[ \sum_{\tau=t}^{T_1} r\left(C_{\tau,a}\right) \mid S_t, A_{t,\star} \neq a \right]
    \]
    does not depend on the recommendation decision $a$.
\end{theorem}
The generalized stickiness model $V^{(a)}_{\pi^0}(z_a;  s)$  in the theorem measures expected long-term reward associated with a user's engagement with item $a$. If $r(c)=\ind(c>0)$ is a binary indicator of consumption of the item, then this generalized definition aligns with the one presented in \eqref{eq:stickiness_definition} in Section \ref{subsec:stickiness}. 

\subsubsection{Extra modeling choices yielding Corollary \ref{cor:Q-function-prototype}}
We introduce two further assumptions. It should be emphasized that Theorem \ref{thm:main} does not impose these assumptions.
They are used only in a special case of the result in Corollary \ref{cor:Q-function-prototype}, which mirrors the prototypes in Section \ref{sec:prototypes}. Under the next assumption, it is enough to track a binary indicator of whether a user engages with an item in a given period, discarding the continuous outcome of their engagement time. 
\begin{assumption}[Binary outcomes]\label{assumption:bonus1}
    The reward function is a binary indicator of positive consumption, written $r(c)=\ind\{c>0\}$. The function $f(z,c)$ defined in \eqref{eq:content-state} maps all positive consumption levels to the same value, meaning $f(z,c)=f(z, \ind(c>0))$. 
\end{assumption}
When this assumption holds, the rule \eqref{eq:content-state} for updating the content-engagement state can be rewritten as
\begin{equation}\label{eq:content-state-binary}
    Z_{t+1,a} = f(Z_{t,a}, \ind\{ C_{t,a}>0 \}).
\end{equation}
The exponential moving averages in Example \ref{ex:ema-content-state} provide a concrete illustration. 

The next assumption lets us replace the raw state variable $S_t$ appearing in Theorem \ref{thm:main}---which represents the full data available on the user---with concise state variables introduced in Section \ref{subsec:pragmatic-user-state}.  
\begin{assumption}[Sufficiency of state variables]\label{assumption:bonus2}
  Under $\pi^0$, for any item $a\in \Abb_\star$ and period $t$,
    \[
     \mathbb{P}_{\pi^0}(C_{t,a}> 0 | S_t, A_{t,\star}) = \mathbb{P}_{\pi^0}(C_{t,a}> 0 | u_t, X_t, Z_{t,a}, A_{t,\star}) 
    \]
  and
    \[
    \mathbb{E}_{\pi^0}\left[\sum_{\tau=t+1}^{T_1}  r(C_{\tau,a}) \mid  S_{t}, Z_{t+1,a}\right]
    =\mathbb{E}_{\pi^0}\left[\sum_{\tau=t+1}^{T_1}  r(C_{\tau,a}) \mid  u_t, Z_{t+1,a}\right].
    \]
\end{assumption}

\section{Online and offline experiments showing real-world impact}
\label{sec:prototypes}
Our theoretical framework, centered on optimizing for long-term user satisfaction through aiding in the formation of lasting item-level listening habits, has been put to the test in real-world implementations at Spotify. These implementations were initially designed to serve as testable prototypes that, even with low engineering effort, might demonstrate the enormous potential impact of purposefully optimizing for the long term in recommendation systems. But the methodology proved so effective that it is now used in production in several parts of the app. We share evidence from A/B tests involving tens of millions of users.

Key findings from our experiments show significant improvements in both targeted and holistic metrics. In a focused A/B test, we directly measured the long-term outcomes of individual recommendations, observing massive increase in habitual listening to the recommended items.
A larger-scale, persistent experiment revealed that our policy, when applied consistently, leads to notable enhancements in overall app-level outcomes. These results are particularly noteworthy in an industry where myopic, short-term optimization is often the norm. 

For product reasons, two cases are treated separately. Podcast \emph{discovery} recommendations (Section~\ref{sec:discovery-proto}) are restricted to podcast shows the user has never tried previously. Podcast \emph{resurfacing} recommendations (Section \ref{sec:resurfacing-proto})  focus on shows the user has already listened to previously.
It is important to note that our methodology is unified. The logic governing discovery recommendations is essentially a special case of the resurfacing logic.

Both discovery and resurfacing recommendations yielded promising offline results. However, the discovery use-case was prioritized for initial A/B testing due to its larger potential impact and simpler engineering requirements\footnote{Resurfacing recommendations have limited benefit for users with minimal podcast exposure and pose greater engineering challenges. They require near-real-time systems to maintain engagement states and compute stickiness model predictions on-the-fly, incurring significant infrastructure costs.}.  The success of the discovery recommendation A/B test led to concentrated efforts on expanding this approach. While resurfacing recommendations received less immediate attention as a result, they may still offer substantial potential for future applications or in other recommendation scenarios.

\subsection{Podcast discovery} 
\label{sec:discovery-proto}

 We begin in Section \ref{subsubsec:fine-tuned-short-term} by describing, at a somewhat superficial level, how podcast recommendations are optimized for short-term outcomes. Section \ref{subsubsec:discovery_stickiness} then describes, in greater detail, how this methodology was augmented to optimize for longer-term goals. Finally, Sections \ref{sec:discovery-tars} and \ref{sec:discovery-home} present results from two A/B tests. 
 
\subsubsection{Control (or ``incumbent'') policy: optimizing the probability of immediate listening}
\label{subsubsec:fine-tuned-short-term}

For concreteness, consider now the problem of recommending a podcast show on the banner component of the home screen in the Spotify mobile app, displayed in Figure \ref{fig:screenshot-banner}. Given a trained representation of users and items, one simple recommendation strategy is to pick an item with maximal affinity score \eqref{eq:affinity} among a pool of candidates.

The recommendation policy in the control group---what we called the ``incumbent policy'' in our theory---uses a more refined strategy. It consists of training a fine-tuned model that leverages the pre-trained embedding to predict the short-term consequences of a recommendation.
Consider a dataset $\Dscr = \{(\nu_{A}, u, X, \ind(Y >0) ) \}$ describing the immediate outcome of past recommendations on the banner.
This contains the vector representation $\nu_{A}$ of the recommended item, the vector representation $u$ of the user at the time of the recommendation, a vector $X$ describing the user's context (e.g. the time of day), and an indicator $\ind(Y>0)$ of whether the user clicked and listened to the item.
We might learn the weights $w$ of a neural network $P_{w}\left(\nu, u, x\right) \in (0,1)$ by minimizing the prediction error in predicting the indicator $\ind(Y >0)$ on $\Dscr$.
Specifically, they are chosen to approximately solve:
\begin{equation}\label{eq:short-term-training}
\min_{w} \sum_{(\nu, u, x, \ind(Y>0)) \in \Dscr} {\rm CrossEntropyLoss}\big[ P_{w}\left(\nu, u, x\right) \, ,\, \ind(Y>0)  \big],
\end{equation}
where ${\rm CrossEntropyLoss}(p, q) = -p \log q - (1 - p) \log (1 - q)$.
Given trained weights $\hat{w}$, a myopic recommender selects the item $\argmax_{a\in \Abb_{\star}} P_{\hat{w}}(\nu_{a}, u_t, X_t)$ that the user is most likely to listen to among a pool of eligible candidates $\Abb_{\star}$.

\subsubsection{Treatment policy: optimizing for lasting discoveries}\label{subsubsec:discovery_stickiness}
Over the same list of eligible items, the treatment policy selects the one which maximizes the long-term value measure $Q_{\pi^0}(S_t, a)$, leveraging a variant of the formula in Corollary \ref{cor:Q-function-prototype}.  

The focus on discovery recommendations means we only need to consider two possible content relationship states: 1) the state $\vec{0}$, representing never having listened to $a$ previously and b) the state $z^+ \equiv \alpha \equiv f(1, \vec{0})$ representing a user who listened for the first time yesterday. Focusing on these two states, we rewrite the formula in Corollary \ref{cor:Q-function-prototype} as, 
\begin{equation}
\begin{aligned}
\label{eq:true-qval-discovery}
Q_{\pi_0}(S_t, a) - b_{\pi^0}(S_t) =
& \underbrace{\Big[ \mathbb{P}_{\pi^0}(C_{t,a}> 0 | u_t, X_t, Z_{t,a}=\vec{0}, A_{t,\star}=a) -   \mathbb{P}_{\pi^0}(C_{t,a}> 0 | u_t, X_t, Z_{t,a}=\vec{0}, A_{t,\star}\neq a)\Big]}_{\text{impact of recommendation on  probability of listening}}  \\
& \qquad\qquad \times \underbrace{\left[ \left(1 + \gamma V^{(a)}(u_t, z^+) \right) - \left(0 + \gamma V^{(a)}(u_t, \vec{0}) \right)  \right]}_{\text{impact of listening on long-term engagement}}.
\end{aligned}
\end{equation}
To implement this formula, one must model how recommendations change the probability of listening \emph{to items the user the user has never previously listen to}, and how listening to an item for the first time increases projected future activity with the item (i.e. changes stickiness). 

We now describe how the treatment policy separately estimates the short-term engagement probabilities and long-term stickiness predictions.  Each involves some approximations that are designed to faithfully capture the essential logic in \eqref{eq:true-qval-discovery} while minimizing engineering effort and simplifying internal communication.  

\paragraph{Counterfactual terms are small for discovery recommendations.} The  formula in \eqref{eq:true-qval-discovery} evaluates the \emph{counterfactual} value of a recommendation. To calculate it, one should model the chance a user listens if the item is recommended and subtract from that the chance of the user listens organically, without a recommendation. Later, when we look at resurfacing recommendations---which may recommend items the user already listens to habitually---the counterfactual nature of this evaluation will be essential. 

But for discovery recommendations  (which, recall, are restricted to items the user has never listened to previously), this counterfactual reasoning can be simplified. Given the vast number of items in the corpus, a user is unlikely to organically discover and listen to a specific new item, for the first-time, today. As a result, discoveries that occur via a recommendation can be (essentially) fully credited to the recommendation itself. 

To illustrate this, consider the results in Figure \ref{fig:tars-holdback}. It compares podcast discovery outcomes among users who received personalized promotions of new podcast shows with those in a holdback group, for whom promotions were generated according the same policy, but never shown. Users in the holdback group are targeted with the same number of recommendations (last row of Figure \ref{fig:tars-holdback}), suggesting that randomization is successful, without bias. But users in the holdback group were an order of magnitude less likely to subsequently discover (i.e. listen to) the show selected for them. 

\begin{figure}[t]
    \centering
    \includegraphics[width=.75\textwidth]{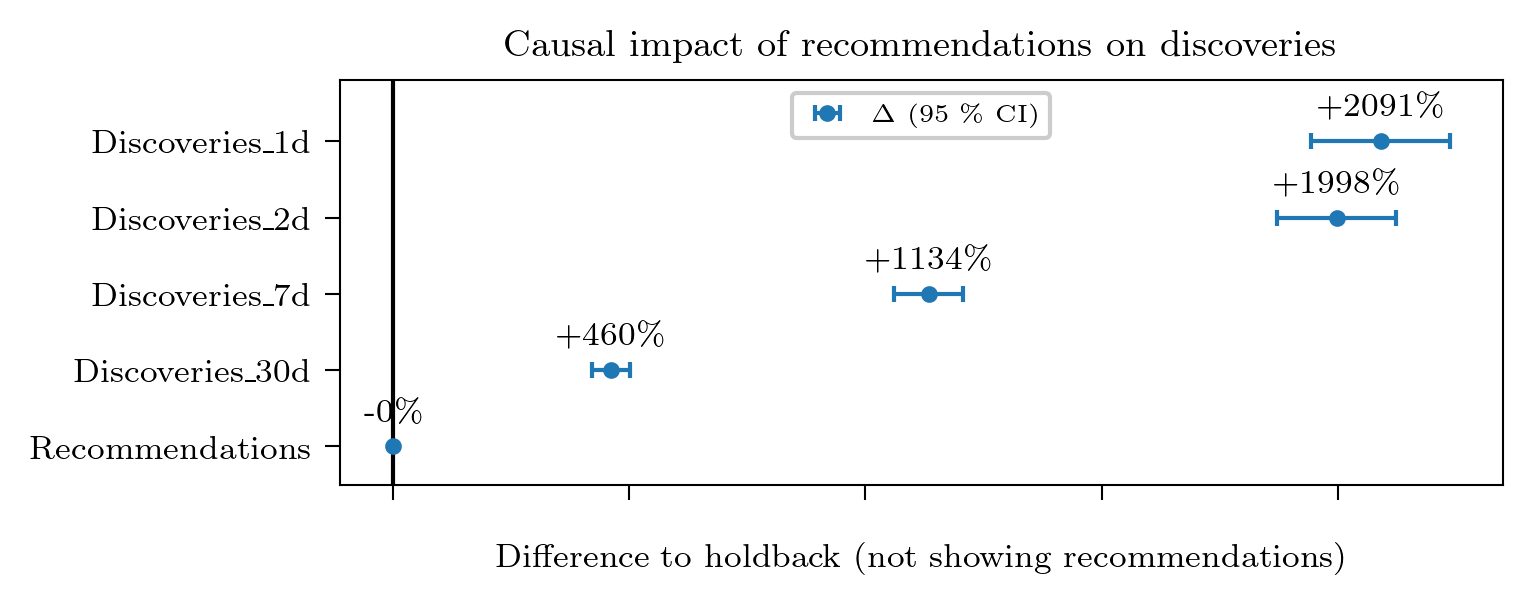}
    \caption{Nearly all discoveries from podcast recommendations can be causally credited to the recommender system.}
    \label{fig:tars-holdback}
\end{figure}

\paragraph{Estimating a short-term model.}
Since, for discovery recommendations, certain counterfactual terms are small, we  approximate the impact of a recommendation on the user's consumption probability (the first term in \eqref{eq:true-qval-discovery})  by the estimated chance  $P_{w}(\nu_a, u_t, X_t)$
 that user listens to the recommendation:
    \[
    \mathbb{P}_{\pi^0}(C_{t,a}> 0 | u_t, X_t, \vec{0}, A_{t,\star}=a) -   \mathbb{P}_{\pi^0}(C_{t,a}> 0 | u_t, X_t, \vec{0}, A_{t,\star}\neq a) \approx P_{w}(\nu_a, u_t, X_t).
    \]
    With this approximation,  the treatment policy uses the same short-term estimates as the control (or ``incumbent'') policy, differing only because of the stickiness models. Bringing the treatment policy closer to the control policy in this manner makes it easier to communicate the results of the test to internal stakeholders. 

    \paragraph{Estimating the long-term stickiness models.}

    Now we describe practical estimation of the stickiness models in \eqref{eq:true-qval-discovery}. First, recall that our theoretical modeling introduced an implicit discount factor $\gamma$ by assuming that users churn with constant probability. This is not true in practice, so any discount factor  needs to be chosen in an ad-hoc manner. Rather than choose $\gamma =59/60$, corresponding to an ``effective horizon'' of $1/(1-\gamma)=60$ days, we use a fixed-horizon of 60 days\footnote{%
    We choose a 60-day window for practical reasons.
    It is long enough to capture recurring podcast listening habits, but short enough to measure in practical experiments.} and drop the discounting.  That is, we set $\gamma=1$ in \eqref{eq:true-qval-discovery} and, overloading notation, define the  stickiness of a discovery as    \begin{equation}\label{eq:discovery_stickiness_prototype}
    V_{\pi^0}^{(a)}\left(   z^+ ; u_t \right) =\mathbb{E}_{\pi^0}\left[ \sum_{\tau=t+1}^{t+60} \ind\left(C_{\tau,a}>0\right)
      \mid Z_{t+1,a}=z^+ \, ,\, u_t=u, S_{t+1} \neq \varnothing  \right],
    \end{equation}
a fixed-horizon variant of the definition in \eqref{eq:item-level-value-func}.

To implement \eqref{eq:true-qval-discovery}, one needs to model the stickiness $V_{\pi^0}^{(a)}(u, z)$ for any user vector $u$, any item $a$ and but just two possible content-engagement states $z$. For now, focus on approximating stickiness from the engagement state $z^+$ appearing in \eqref{eq:true-qval-discovery}, representing a user who just listened the item for the first time. We train a new vector representation of items $\{\theta_a \}_{a\in \mathbb{A}}$, called ``stickiness'' vectors,  that (to reduce engineering effort) pair with pre-existing user vectors, and approximate stickiness of new discoveries through dot products as:
\[
V_{\pi^0}^{(a)}(u, z^+) =  u^\top \theta_a.
\]
Recall from~\eqref{eq:affinity} that, through the dot product $\nu_a^\top u$, the original item vector $\nu_a$ encodes a user's propensity to have some interaction the item, however brief it may be.
At times, we call these ``item clickiness vectors.''
In contrast, the taking a dot products with an item stickiness vector $\theta_a$ encodes the propensity of users with certain tastes to return to item $a$ many times after listening to it for the first time.  

Since we use pre-existing user-representations (rather than co-learning them), item stickiness vectors are easy to estimate. For each item $a$, we gather a dataset on users who discovered item $a$, their features at the time they first listened to the item, and their interactions with the item post-discovery. 
Specifically, in our prototypes, we set $\Dscr_a= \{(u, \hat{V}^{(a)} ) \}$ to be a dataset of tuples ranging over users who discovered the item more than 60 days ago. That is, we locate users who are entering engagement-state $z^+$. 
The tuple $(u, \hat{V}^{(a)})$ describes the vector representation of the user's tastes $u\in \mathbb{R}^d$ at the time of the discovery and  $\hat{V}^{(a)}\in  \{0, \cdots, 59\}$  is the sum of consumption indicators in \eqref{eq:discovery_stickiness_prototype}, equal to the number of days in the subsequent 59 days that the user returned to and listened to the item. Item ``stickiness'' vectors $\theta_a$ are trained by solving the least-squares\footnote{In practice, we solve a regularized least-squared problem. Regularizing toward an informed prior mean lets us handle long-tail content, for which limited data is available.} problem
\[
\theta_a  = \argmin_{\theta \in \mathbb{R}^d} \sum_{(u, \hat{V}^{(a)})\in \Dscr_a} (\theta^\top u - \hat{V}^{(a)})^2.
\]

Stickiness vectors are trained not just on recent discoveries from banner recommendations, but on discoveries from across the application and over a fairly long time window.
We gather enough data to do more than fine-tuning (see Remark \ref{rem:stickiness_is_not_finetuning}), instead offering a new understanding of items in the catalog.
This \emph{data-pooling} is a key advantage to training the long-term model separately from the short-term one. Abstractly, it is Assumption \ref{assumption:markov-in-time} that enables this separation.

This process could be repeated to estimate a different set of stickiness vectors from the null state $z=\vec{0}$, representing users who have never tried the item. In practice, those stickiness estimates turn out to be much smaller, an do not meaningfully impact the full evaluation in \eqref{eq:true-qval-discovery}. Again, the intuition is that, given the vast number of items in the corpus, users are unlikely to organically discover and attach to this specific item. For simplicity, we round $v_{\pi^0}^{(a)}(u, \vec{0})$ to $0$, avoiding the need to track this term.

\begin{remark}[New representations, rather than fine-tuning]\label{rem:stickiness_is_not_finetuning}
    We described the short-term models as a ``fine-tuned'' model built on top of foundational learned representations. They rely on foundational feature representations of users and items, which are produced by solving a different prediction problem, combine them with other features, and use this information to predict the probability of a listen given an recommendation. Fine-tuning requires less data on each specific item, so such a model can be trained on recent data, trained specifically to an individual surface (e.g. banner promotions), and responsive to context.    The stickiness vectors can instead be viewed as a new foundational representation. They offer an understanding of each specific item in the catalogue. The idea is that item ``clickiness vectors'' $\nu_a$ do not reflect what makes an item ``sticky'', and for whom it is sticky. Learning a new representation requires lots of data, so it is critical that we pool across surfaces and long-time horizon. 
\end{remark}

\paragraph{Final implementation.} Combining the short- and long-term models above, we derive a faithful approximation of the true formula \eqref{eq:true-qval-discovery} that is very easy to implement and describe to stakeholders. Recommendations are chosen by maximizing the approximate score $\hat{Q}(S_t, a)$ given by 
\begin{equation}\label{eq:discovery_stickiness_implementation}
\hat{Q}(S_t, a) = \underbrace{P_{w}(\nu_a, u_t, X_t)}_{\text{clickiness}} \times \left( 1 +\underbrace{u_t^\top \theta_a}_{\text{stickiness}}\right).
\end{equation}
The term we call clickiness is an existing, heavily optimized short-term model that predicts the probability of a users listens to a new item given it is recommended. Dot products between user vectors  and the newly trained stickiness vectors capture whether a user with given tastes is expected to return to this item many days in the future. It is worth emphasizing that the  additional simplifications in  \eqref{eq:discovery_stickiness_implementation} is due to the focus on discovery recommendations, where some terms in the general formula \eqref{eq:true-qval-discovery} are small enough to be ignored.

\subsubsection{Experimental impact on targeted metrics in a small A/B test}
\label{sec:discovery-tars}

The first A/B test we describe involves results from show promotions on the banner component displayed in Figure \ref{fig:screenshot-banner}.
Twelve markets were included in the test.
For each one, editors hand-picked between 3 and 74 shows that are original and exclusive to Spotify.
The test period lasts one week.
During this time, a single personalized recommendation of an eligible show is made to each user in the test.
For every user, we exclude from the candidate pool any show that the user has already listened to in the past.
Downstream activity is observed for sixty days after the recommendation was first seen by the user.
Users are assigned to one of four groups randomly.
\begin{description}
\item[Control (or ``the incumbent policy'').] In the control group, a recommendation is made to maximize the predicted probability of a listen as in Section \ref{subsubsec:fine-tuned-short-term}.

\item[Personalized.] In the long-term, personalized treatment group, recommendations are determined by maximizing an estimate of the $Q$-value in \eqref{eq:discovery_stickiness_implementation}. This is the treatment group of primary interest. 

\item[Unpersonalized.] The long-term, unpersonalized treatment group uses the same personalized short-term model, but predicts stickiness of item $a$ via an unpersonalized estimate $\bar{v}_a = \lvert \Dscr_a \rvert^{-1} \sum_{(u, R) \in \Dscr_a} R$, instead of estimating vectors $V_a$ that enable personalized predictions.

\item[Square-root.] The long-term, square-root, unpersonalized treatment group maximizes the predicted probability of a listen times the square-root of unpersonalized stickiness.
The square-root is a heuristic transformation that de-emphasizes stickiness differences across items.
\end{description}

For approximately 37\% of users, each of the four recommendation policies select the same item at the time of the impression.
(A contributing factor is that the pool of candidates in a given market may be small.)
We say that the other 63\% of users are \emph{impacted} by their randomized assignment to a treatment  group.
For each user that receives a recommendation, we measure three quantities.
\begin{description}
\item[First streams (or ``discoveries'').] Whether the user listened to the recommendation or not.
We use a 30-second threshold to determine whether a listen has occurred or not. 

\item[60-day activity.] The number of days the user is active with the recommended show in the 60-day window that starts from the day the user receives the recommendation.

\item[60-day minutes.] The total number of minutes the user listened to the recommended show in the same 60-day window.
\end{description}
Figure \ref{subfig:discovery-tars-core} shows average treatment effects among impacted users.
There are three main findings:

\begin{figure}[t]
\centering
\begin{subfigure}{\textwidth}
    \centering
    \caption{User metrics}
    \vspace{3mm}
    \label{subfig:discovery-tars-core}
    \includegraphics{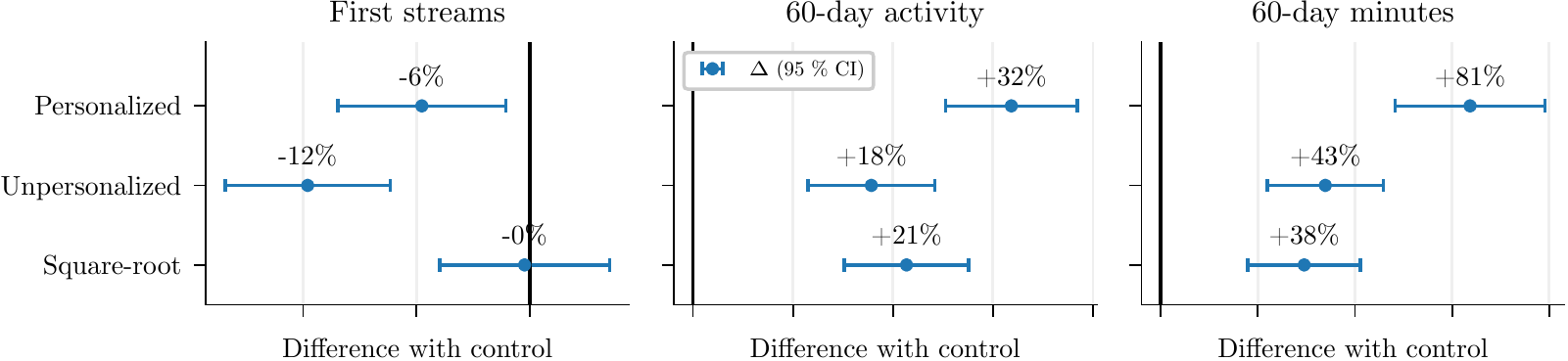}
\end{subfigure}\\
\vspace{5mm}
\begin{subfigure}{.5\textwidth}
    \centering
    \caption{60-day minutes, normalized by Control}
    \vspace{3mm}
    \label{subfig:discovery-tars-distribution}
    \includegraphics{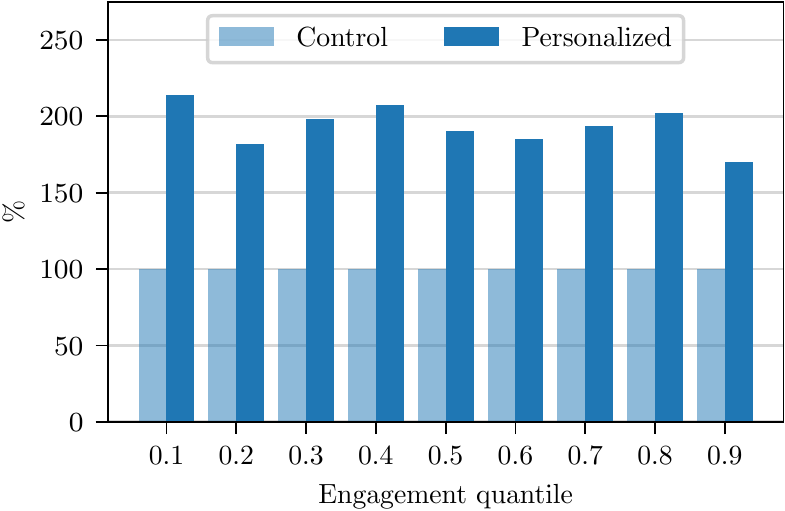}
\end{subfigure}\hspace{15mm}%
\begin{subfigure}{.4\textwidth}
    \centering
    \caption{Calibration of long-term value}
    \vspace{3mm}
    \label{subfig:discovery-tars-calibration}
    \includegraphics{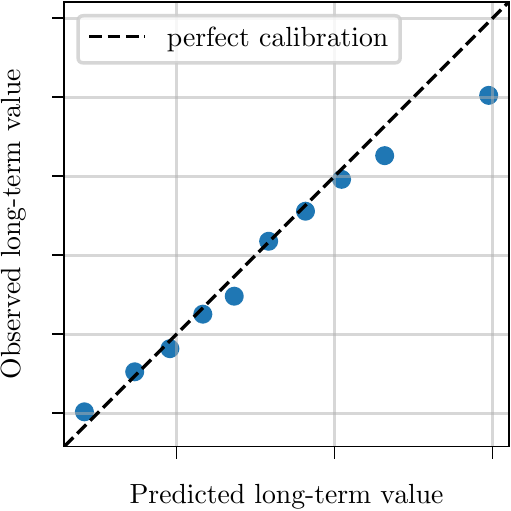}
\end{subfigure}
\caption{Results from an online experiment on the Home banner.}
\label{fig:banner_experiment}
\end{figure}

\begin{enumerate}
\item Optimizing for the long-term is different from optimizing for the short-term.
The long-term recommendation policy largely selected different items than control.
The number of first-streams of the recommended item, which is purposefully optimized by control, was reduced by optimizing for the longer term.

\item Explicitly grappling with the long-term goal leads to large impact.
Relative to the control group, users in the long-term personalized treatment had an 81\% increase in 60-day show minutes and a 32\% increase in 60-day show-active days.
Since these are average treatment effects, one might be concerned that it is driven by a small cohort of users that listen for a disproportionately long time.
In fact, Figure \ref{subfig:discovery-tars-distribution} reveals that the treatment cell created many more moderately engaged users.
For instance, that figure shows that median 60-day listening minutes were more than 80\% larger in the long-term personalized group than in the control group.

\item Causal assumptions were validated experimentally.
The stickiness vectors described in Section \ref{subsubsec:discovery_stickiness} were trained using data from users who discovered the items anywhere on Spotify, including search. 
Our methodology assumed, implicitly, that users who discovered the show through an advertisement on the banner would have similar downstream behavior. (This is implied by Assumption \ref{assumption:markov-in-time})
Figure \ref{subfig:discovery-tars-calibration} confirms this.
The long-term predictions of our models are almost perfectly calibrated to the results observed in the test.
\end{enumerate}

\subsubsection{Impact on holistic, app-level metrics in  larger and longer experiment}
\label{sec:discovery-home}

A second experiment tested a larger and more persistent change in the recommendation policy.
The horizontal row of recommendations labeled ``Shows you might like'' in Figure \ref{fig:screenshot-shelf} is called a \emph{shelf}.
Each item is called a \emph{card}.
The first two cards of the shelf are immediately displayed on the screen, and users can scroll to the right in order to reveal more cards.
The experiment ran for nine weeks, and during this time, the shelf was pinned near the top of the screen, as illustrated in Figure~\ref{fig:screenshot-shelf}.
Two aspects distinguish this experiment from the one described in Section~\ref{sec:discovery-tars}.
First, in this experiment tens of thousands podcast shows are eligible to be recommended on the shelf.\footnote{%
A candidate generation model returns a smaller, personalized pool of approximately a hundred candidate shows for each user.
In our experiment, this model is fixed, and our intervention focuses on a second stage where the candidates are ranked.}
As such, different recommendation policies almost always result in a different ranking.
Second, whereas the Home banner experiment studied the impact of a single recommendation, in this experiment we rank multiple recommendations at once.
We display the shelf every time the user opens the Home screen, over the entire duration of the test.
As a consequence, users have many opportunities to interact with the recommendations, and we expect this experiment to exhibit a significantly larger impact on users' overall podcast habits.
The Spotify \href{https://engineering.atspotify.com/2020/10/spotifys-new-experimentation-platform-part-1/}{experimentation platform} was used to assign users to one of three groups, uniformly at random, and assignments are fixed through the test period.  
\begin{description}
\item[Control.] Users in the control group received recommendations that were ranked in decreasing order of predicted listening probability, similar to Section \ref{subsubsec:fine-tuned-short-term}.

\item[Personalized.] Users in the personalized long-term treatment group received recommendations that were ranked in decreasing order of the $Q$-values in Section \ref{subsubsec:discovery_stickiness}.

\item[Square-root.] Users in this group received recommendations ranked in decreasing order according a predicted impression-to-listen probability times the square-root of a personalized stickiness estimate, similarly to the corresponding group in Section~\ref{sec:discovery-tars}.
\end{description}

Given that the intervention we study encompasses a large number of recommendation over a long period of time, we measure user metrics that capture users' overall engagement with podcasts on the platform.
We focus on the following three metrics.
\begin{description}
\item[Overall consumption in Week 8.] The total number of podcast minutes the user listened to during the week leading to day 56 after first exposure.

\item[Consumption from discoveries.] The total number of minutes the user listened to podcasts discovered anywhere on the app during the first six weeks after first exposure.

\item[At least one lasting discovery.] Whether the user discovered, anywhere on the app, at least one show that they listened to
\begin{enuminline}
\item on at least three separate days, and
\item for at least 2 hours in total,
\end{enuminline}
during the first six weeks after first exposure.
\end{description}

Altering the ranking policy for that single shelf had a substantial impact, as displayed in Figure~\ref{fig:discovery-home-holistic}.
Week-8 overall podcast listening minutes on the app were 1.7\% greater among users in the long-term, personalized treatment group.
Consumption from podcasts discoveries in the first six weeks increased by 6.2\%.
The number of users who had a ``lasting'' podcast discovery during that time period increased by 5.4\%.
These results suggest that the insights obtained using the clean but small intervention studied in Section~\ref{sec:discovery-tars} carry over to practical, complex recommender systems:
Explicitly driving podcast recommendations towards long-term engagement leads to substantial impact.

In Figure~\ref{fig:discovery-home-timeseries} we plot the impact of the two treatment group on the weekly total number of podcast minutes consumed, as a function of the number of weeks since users' first exposure.
The value at week 8 is identical to the metric displayed in Figure~\ref{fig:discovery-home-holistic} (left).
This time series highlights the fact that the rewards of successful recommendations early in the test take many weeks to fully materialize.
Measuring the effect of long-term policies at steady state requires long-running experiments, and naive metrics measured early on might severely underestimate the true long-term impact.

\begin{figure}[t]
\centering
\begin{subfigure}{\textwidth}
    \centering
    \caption{User metrics}
    \vspace{3mm}
    \label{fig:discovery-home-holistic}
    \includegraphics{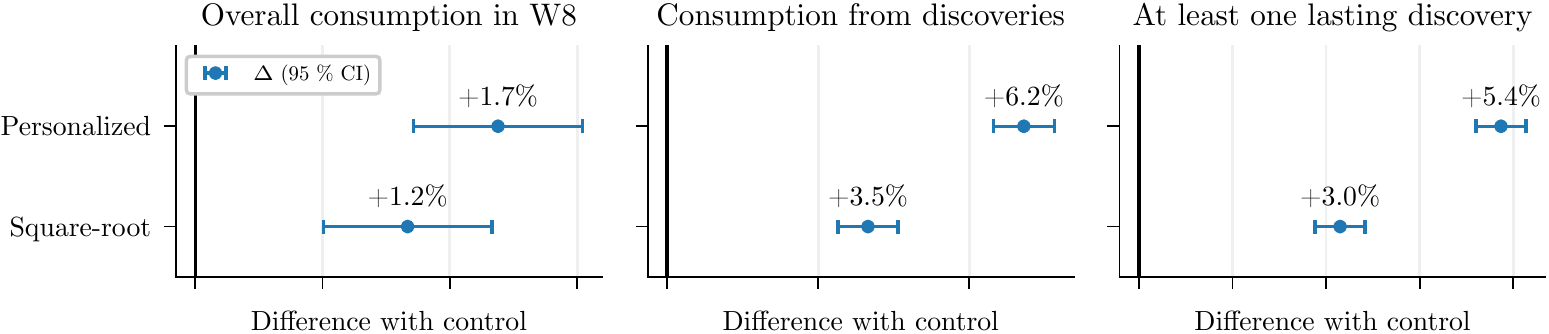}
\end{subfigure}\\
\vspace{5mm}
\begin{subfigure}{\textwidth}
    \centering
    \caption{Overall podcast consumption in $n$th week}
    \vspace{3mm}
    \label{fig:discovery-home-timeseries}
    \includegraphics{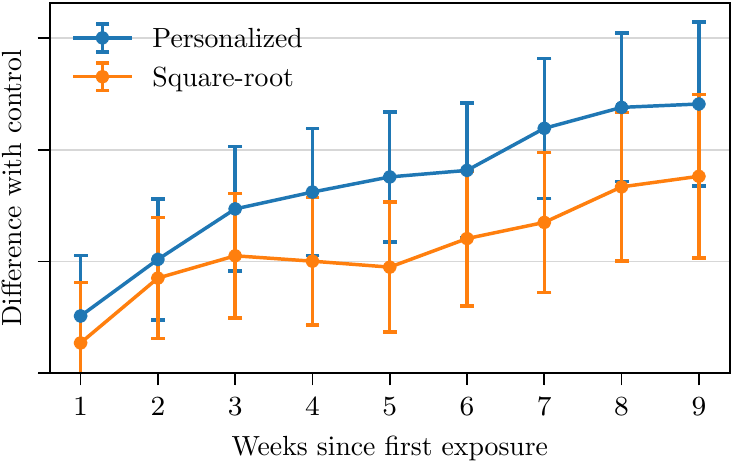}
\end{subfigure}
\caption{Results from an online experiment on the Home podcast discovery shelf.}
\label{fig:discovery-home}
\end{figure}

\subsection{Podcast resurfacing} 
\label{sec:resurfacing-proto}

We have presented a real-world implementation that optimizes recommendations of unfamiliar items---those the user has never previously listened to. This service helps users discover items they enjoy---signaled by their choice to return to it later. Here we present empirical results for the problem of impactful resurfacing items the user tried previously. We again rely on the formula for the $Q$-value presented in Corollary \ref{cor:Q-function-prototype}, which applies seamlessly to both discovery of new items and resurfacing of familiar ones. 

As discussed at the beginning of the section, we present offline correlational results that hint at the promise of the general methodology, but leave open the question of whether this will bear out experimentally, as it did in discovery case.

\subsubsection{Offline analysis}
\label{sec:resurfacing-eval}

In order to build a qualitative understanding of different approaches to recommending familiar content, and to illustrate the opportunities that our counterfactual, long-term approach opens up, we analyze recurring interaction data of Spotify users with three podcast shows. 

For purposes of illustration, we ignore the user taste representation $u_t$ and the context $X_t$, emphasizing the role of content-relationship states. The $Q$-value formula from Corollary \ref{cor:Q-function-prototype} can then be rewritten as
\begin{equation}
\begin{aligned}
\label{eq:resurfacing-qval-simplified}
Q_{\pi^0}(S_t,a) =
& \underbrace{\Big[\mathbb{P}_{\pi^0}(C_{t,a}> 0 | Z_{t,a}, A_{t,\star}=a) -  \mathbb{P}_{\pi^0}(C_{t,a}> 0 | Z_{t,a}, A_{t,\star}\neq a)  \Big]}_{\text{impact of recommendation on listening probability}}  \\
& \qquad\qquad \times \underbrace{\left[ \left(1 + V^{(a)}_{\pi^0}(Z_{t+1,a}^{+}) \right) - \left(0 + V^{(a)}_{\pi^0}(Z_{t+1,a}^{-}) \right)  \right]}_{\text{impact of listening on long-term engagement}}.
\end{aligned}
\end{equation}

For purposes of illustration, we train non-parametric models on an aggregated representation of the content-relationship state $Z$ for several distinct podcast shows $a$. Details are provided in Appendix \ref{sec:resurfacing-eval-details}. We visualize the patterns learned by these models in Figure~\ref{fig:resurfacing}.
Each plot represents values for a hypothetical user that has been using the Spotify service for 4 weeks.
The $xy$-axes corresponds to a two-dimensional projection of the content-relationship state;
For example, the cell at index $(2, 3)$ corresponds to a user that was active 20\% of the days in weeks 1 \& 2, and 10\% of the days in weeks 3 \& 4.
The value at a given cell is computed by averaging model estimates corresponding to all engagement states with the prescribed activity ratio during the two periods.
Note that the bottom-left corner, i.e., position $(1, 1)$, corresponds to the engagement state associated to the discovery setting (the user has never engaged with this item before).
The plots visualize the following quantities as a function of $z$, for all three podcasts under consideration.
\begin{itemize}
\item Figure~\ref{fig:resurfacing-prob} represents the factual and counterfactual probability of listening, $\mathbb{P}_{\pi^0}(C_{t,a}> 0 | Z_{t,a}, A_{t,\star}=a)$ and $\mathbb{P}_{\pi^0}(C_{t,a}> 0 | Z_{t,a}, A_{t,\star}=a) -  \mathbb{P}_{\pi^0}(C_{t,a}> 0 | Z_{t,a}, A_{t,\star}\neq a)$, respectively.

\item Figure~\ref{fig:resurfacing-val} represents the factual and counterfactual long-term value, $1 + V^{(a)}_{\pi^0}(z^+)$ and $[1 + V^{(a)}_{\pi^0}(z^+)] - [0 + V^{(a)}_{\pi^0}(z^-)]$, respectively, where $z^+ = \alpha \circ z + (1 - \alpha) \circ 1$ and $z^- = \alpha \circ z + (1 - \alpha) \circ 0$.

\item Figure~\ref{fig:resurfacing-qval} represents the $Q$-value~\eqref{eq:resurfacing-qval-simplified}.
This corresponds to the product of the probability of listening and long-term value counterfactuals.
\end{itemize}

\begin{figure}[p]
\begin{subfigure}{\textwidth}
    \centering
  \caption{Probability of listening}
  \label{fig:resurfacing-prob}
  \vspace{2mm}
  \includegraphics{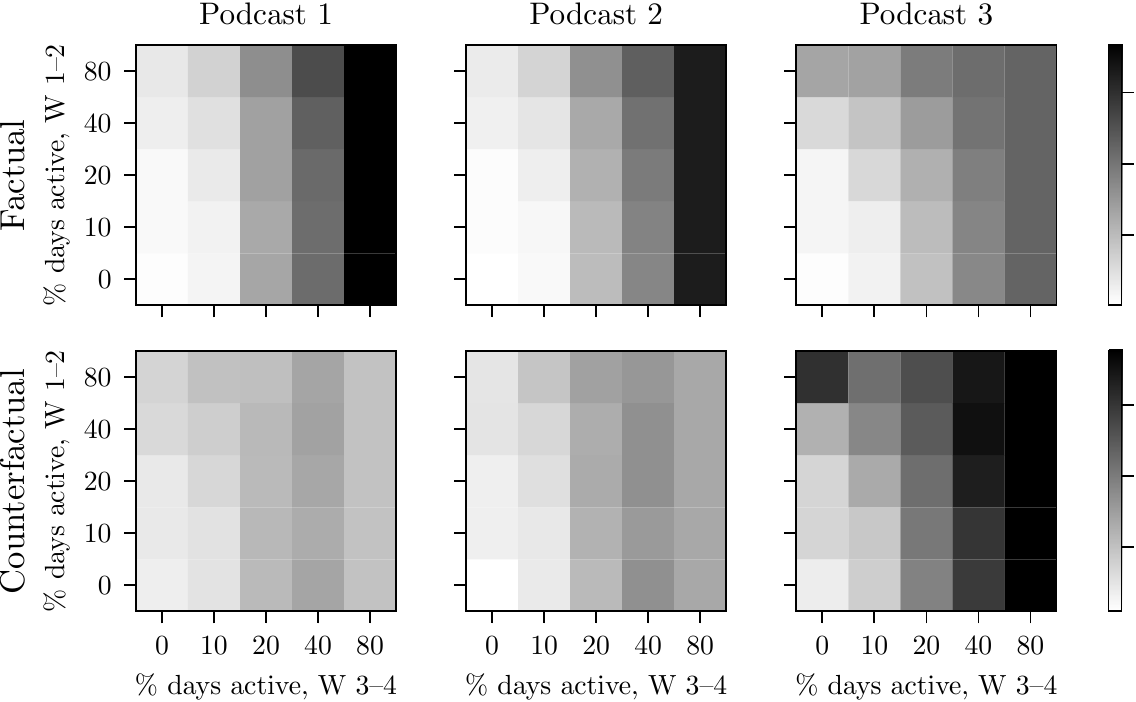}
\end{subfigure}\\
\begin{subfigure}{\textwidth}
    \centering
  \caption{Long-term value}
  \label{fig:resurfacing-val}
  \vspace{2mm}
  \includegraphics{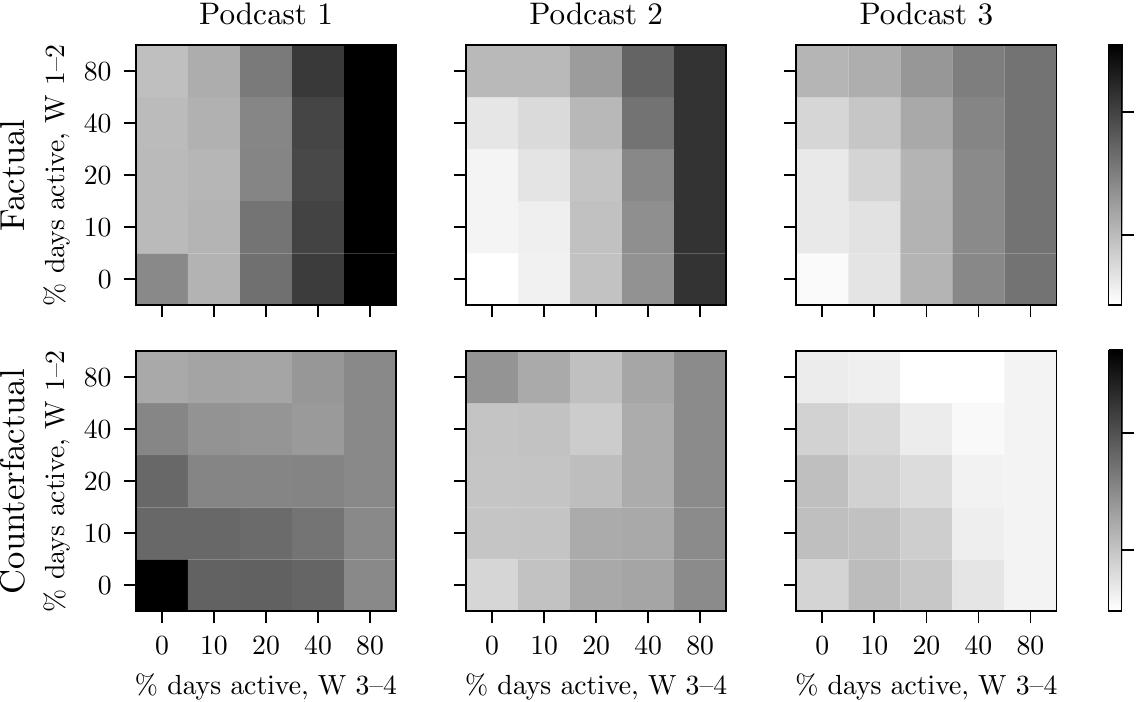}
\end{subfigure}\\
\begin{subfigure}{\textwidth}
    \centering
  \caption{$Q$-value}
  \label{fig:resurfacing-qval}
  \vspace{2mm}
  \includegraphics{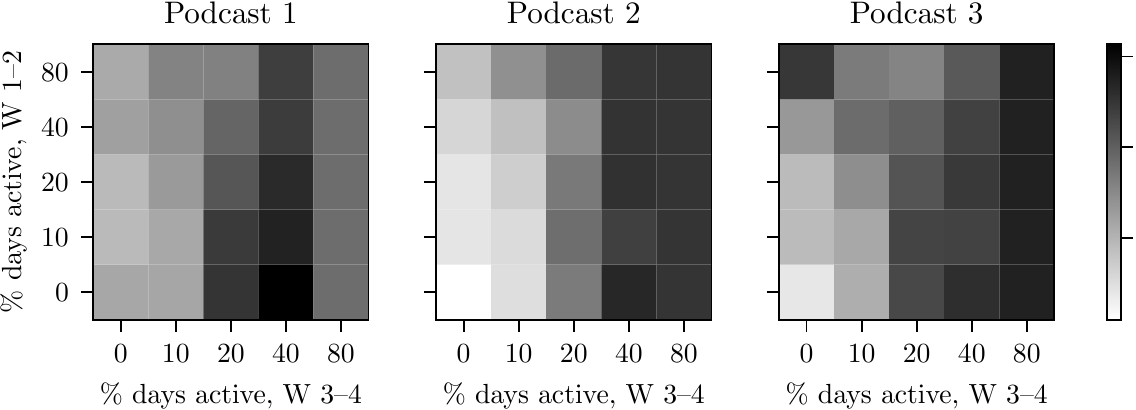}
\end{subfigure}
\caption{Resurfacing quantities.}
\label{fig:resurfacing}
\end{figure}

Collectively, these reveal several important insights.
First, \emph{counterfactual reasoning is important}.
Content that the user listened to recently is systematically more likely to be listened to today, but a lot of these listens would occur organically anyway, even without a recommendation.
A similar observation holds for the long-term value:
High recent engagement is likely to lead to high engagement in the future, irrespective of whether or not a listen occurs today.
In contrast, the counterfactuals (second row in Figures~\ref{fig:resurfacing-prob} and~\ref{fig:resurfacing-val}) highlight that the highest \emph{impact} on short-term and long-term outcomes is often obtained in lower engagement states. This finding differs from the discovery case treated earlier.

Second, there is \emph{significant heterogeneity across items and states}.
Focusing on Figure~\ref{fig:resurfacing-qval}, we observe that, for a given state, there is a large difference in the predicted impact of a recommendation across items.
This is most easily observed for the no-engagement state (bottom-left cell), for example:
A recommendation of Podcast~1 appears to be more valuable than the two other items.
Likewise, for any given item, there are substantial variations in predicted impact across states, and the patterns are not identical across shows.
As a consequence, the relative ranking of items depends on past engagement.
While a recommendation for Podcast 1 has higher value in the no-engagement state (bottom-left cell), a recommendation for Podcast 3 has higher value for users who are in a high past-engagement, low recent-engagement state (top-left cell).
Similar observations hold for Figures~\ref{fig:resurfacing-prob} and~\ref{fig:resurfacing-val}.

Third, the \emph{short-term counterfactual and full $Q$-value are misaligned}.
Reasoning about the downstream impact of short-term outcomes often leads to distinctly different decisions.
For example, Figure~\ref{fig:resurfacing-prob} seems to suggest that recommending Podcast 3 to users with high recent engagement is always the best course of action, by a large margin.
In fact, Figure~\ref{fig:resurfacing-qval} reveals that, once long-term impact is accounted for, there are item-state pairs that might lead to recommendations with a larger impact.

\subsubsection{Open questions around experimental evaluation}
\label{sec:resurfacing-open}

The offline analysis in Section~\ref{sec:resurfacing-eval} helps justify the most important aspects of our methodology.
In particular, the analysis highlights the importance of optimizing for the long-term, of counterfactual evaluation, and of capturing heterogeneity across both items and user relationship states.
For discovery recommendations---a special case of the general methodology---A/B tests described in Sections~\ref{sec:discovery-tars} and~\ref{sec:discovery-home} validate causal assumptions and show large product impact. As discusses at the start of this section, 
 we do not yet have such evidence for recommendations of familiar items.

\section{Revisiting attribution, coordination, and measurement challenges}\label{subsec:challenges-revisited}

We now describe how our methods cope with the challenges  described in Section \ref{subsec:challenges}.  We discuss attribution and coordination challenges only at a high level, mentioning some practical ways the real implementations deal with these issues. 

Appendix \ref{sec:policy_improvement_theory} reviews broader theory of methods that model and optimize the same $Q$-function. This theory suggests some ability to cope with attribution and coordination challenges is due to this high-level goal, rather than our more specialized modeling of the $Q$-function in Section \ref{sec:structural}.

However, our experience suggests that this more specialized modeling played an essential role in overcoming measurement challenges (i.e. in enabling data-efficient solutions).  To support this, Section \ref{sec:samplesize}  provides  an empirical study of data requirements under several $Q$-function estimators.   

\subsection{High-level discussion}\label{subsec:challenges_revisited_intuition}

\paragraph{Attribution.} As discussed in Section \ref{subsec:challenges}, it is subtle to attribute credit for long-term user outcomes to individual recommendation decisions. 
The formula in Corollary \ref{cor:Q-function-prototype} deals with this issue in an intuitively compelling way. Individual recommendations are credited with \emph{changes} in predicted item-level listen habits, and receive no credit for listening that would be anticipated regardless. More abstract theory of reinforcement learning also clarifies a sense in which $Q$-functions coherently attribute credit among individual actions, assuming they are estimated perfectly; see Appendix \ref{subsec:attribution_theory} for a review.  

\paragraph{Coordination.}

The formula in \eqref{eq:resurfacing-qval} aligns conveniently with an organization in which individual teams maintain their own short-term models of user interactions (e.g. responses to banner promotions in Figure \ref{fig:screenshots}, or search results). In practice today, long-term stickiness models which are maintained by a centralized team and are used to augment these distinct short-term models. 

Notice also that definition of the partial $Q$-function defined in \eqref{eq:partialQ} and the corresponding policy improvement goal seem to inherently involve decentralized logic.  It directs an individual team (e.g. the one in charge of banner promotions) to take actions which improve long-term outcomes \emph{in the context of the current incumbent policy controlling the recommender system}. Section \ref{subsec:coordination_theory} reviews a precise, mathematical, interpretation of this policy improvement update as (implicitly) performing coordinate ascent in the space of a policies, an approach ideally suited to decentralized systems.

\paragraph{Measurement.}  
Our domain-specific modeling enhances data-efficient estimation of the Q-function in \eqref{eq:partialQ} through two key aspects. First, we focus on the formation of item-level listening habits. Instead of correlating recommendations with overall `rewards' accumulated over a user's lifetime of app interactions, we measure how a recommendation contributes to habits specific to the recommended item. This focused measurement allows us to capture outcomes directly related to the recommendation, minimizing noise from unrelated user behaviors. Second, our modeling enables the use of distinct datasets and features for short-term engagement and long-term stickiness models. To address the higher variance in long-term outcomes, our stickiness models aggregate data across extended time periods and interactions with many components of the app. The following subsection empirically demonstrates how these aspects significantly boost data efficiency.

\subsection{Empirical analysis of data efficiency}
\label{sec:samplesize}

In this section, we demonstrate the severe measurement challenges inherent in estimating long-term effects of recommendations and illustrate how our approach dramatically reduces sample size requirements. These challenges are fundamental to optimizing for long-term outcomes in any large-scale recommendation system. The core issue lies in distinguishing the impact of individual recommendations from the natural ``noise'' in user behavior, which, as mentioned in the introduction, is akin to measuring the effect of a single meal on lifetime health. We consider four different value estimation methods, and study their sample complexity using randomized recommendation data collected at Spotify.

\subsubsection{Re-purposing A/B test data to demonstrate challenges in Q-value estimation}
Ideally, we would like to demonstrate the difficulty of statistically distinguishing the Q-values of two different actions using a direct, assumption-free estimation strategy. To illustrate this point, we would ideally construct a direct demonstration by recommending individual podcast shows (the actions $a$) to millions of users who have similar features (i.e., similar states $s$) and tracking downstream metrics. Direct $Q$-value estimation would average a holistic, long-term reward among users receiving each recommendation action; for instance, one might average total minutes of listening across all podcasts over the 60-days post-recommendation. Even at Spotify, randomized data of this type at massive scale is unrealistic (say, millions of recommendations for each specific combination of user cohort and show). Granular personalization and a large catalogue of items to potentially recommend 
means that the effective sample size is not so large. This is one reason we have emphasized measurement challenges.

To come up with a compelling demonstration, we construct a setting that is much simpler than what would be encountered in practice, yet still reveals significant challenges. We repurpose data from the A/B test described in Section~\ref{sec:discovery-tars}, adopting a different perspective. Instead of defining actions as show recommendations, we define them as assignments to either control or treatment policies within the test. In this test, a user received a single podcast show recommendation the next time a banner promotion appeared---the treatment assignment only determines which policy made the recommendation.  Effectively, then, assignment to either the control policy $A$ or treatment policy $B$ is a kind of randomized meta-action, that can be viewed as randomly selecting among two pre-determined potential recommendations targeted at that user. We call A and B ``meta-actions''.

 This approach offers two key advantages that simplify our analysis. First, it allows us to compare two recommendation actions with highly differentiated long-term performance that were randomly made to millions of users. Second, since either met-action is already `implicitly' personalized, it is reasonable to compare unpersonalized variants of the $Q$-values\footnote{In measuring population-averages, the quantity we estimate is $Q(a)= \frac{1}{|U|}\sum_{u \in U} Q(S^u, \pi_{a}(S^u))$ for $a \in \{A,B\}$, where $U$ denotes the set of users included in the test, $S^u$ is the state of user $u$ at recommendation time, and $\pi_{a}(S^u)$ denotes either the control ($a=A$) or treatment ($a=B$) policy.}, denoted $Q(A)$ and $Q(B)$. This greatly reduces sample size requirements.

\subsubsection{Datasets used}
In the following, we call the Control or ``incumbent'' policy the meta-action $A$, and the Personalized treatment policy meta-action $B$.
For each action $a \in \{A, B\}$, we construct two datasets.
\begin{itemize}
\item The first dataset, $\mathcal{D}_a = \{(Y_i, R_i, G_i) : i = 1, \ldots, n_{\max} \}$,  contains outcomes of recommendations of item $i$, where $Y_i \in \{0, 1\}$ is a binary indicator for whether the user engages with the recommendation, $R_i$ is the number of 60-day minutes listened to the recommended show, and $G_i$ is the total 60-day minutes listened to all podcast shows across the entire Spotify service.  
\item The second dataset, $\mathcal{D}'_a = \{ R_j : j = 1, \ldots, 7000 \}$, contains 60-days minutes listened to shows discovered via action $a$. This is meant to resemble\footnote{In practical applications of our methodology, this dataset is created from outcomes of past discoveries across the entire Spotify service (see Section~\ref{subsubsec:discovery_stickiness}). Because the test was run long-ago,  that large dataset is no longer saved, and it very costly to reproduce it. Thankfully, for the purposes of studying sample complexity, we can re-use a subset of the first dataset, imagining it was from an auxiliary source---available prior to the test and independent of the outcomes in $\mathcal{D}_a$. Concerns about distribution-shift in the auxiliary data source are evaluated in Section \ref{sec:discovery-tars}. }
  auxiliary datasets used to estimate the stickiness of show-listening post-discovery in Section \ref{subsubsec:discovery_stickiness}. We chose \num{7000} samples, since this is the median number of historical discoveries used to compute the show-stickiness in the experiment discussed in Section~\ref{sec:discovery-tars}.  
\end{itemize}
One difference here is that we work with listening minutes (i.e. reward $r(c)=c$ in \eqref{eq:item-level-value-func}), whereas our real implementations correspond to a reward function that tracks active listening days (i.e. reward $r(c) = \ind(c>0)$ in \eqref{eq:item-level-value-func}). The A/B test results report both measures.

\subsubsection{Estimation procedures evaluated}
We compare three methods for estimating the $Q$-value of a meta-action $a\in \{A,B\}$. 
\begin{description}
    \item[Holistic, long-term:] This approach aims to capture the broadest impact of a recommendation. It measures changes in overall user behavior across the entire platform following a recommendation. In our demonstration, this method tracks the total minutes of podcast listening across all shows over 60 days after the recommendation. Formally, this estimate is given by $Q_{\text{glt}}(a) = \lvert \mathcal{D}_a \rvert^{-1} \sum_{(Y, R, G) \in \mathcal{D}_a} G$.
    \item[Local, long-term:] This method directly measures the long-term effect of a recommendation on engagement with the recommended item. In our case, it tracks the total minutes users spend listening to the specifically recommended podcast show over 60 days. Formally, this estimate is given by $Q_{\text{llt}}(a) = \lvert \mathcal{D}_a \rvert^{-1} \sum_{(Y, R, G) \in \mathcal{D}_a} R$.
    \item[Our Approach:] This combines short-term and long-term data. It estimates the rate at which recommendations lead to initial engagement, and then multiplies this by a separate estimate of the long-term value of such engagements. The long-term value is estimated using historical data on listening minutes for users who engaged with similar recommendations in the past. This is analogous to \eqref{eq:discovery_stickiness_implementation}. Formally, estimate is given by  $Q_{\text{ours}}(a) = \lvert \mathcal{D}_a \rvert^{-1} \sum_{(Y, R, G) \in \mathcal{D}_a} Y \cdot \bar{v}_a$ where  $\bar{v}_a = \lvert \mathcal{D}'_a \rvert^{-1} \sum_{R \in \mathcal{D}'_a} R$. 
\end{description}
Under the assumptions of our theoretical modeling, all three estimators have the same expected value, but they have different variances. Our approach and the ``Local, Long-term'' approach reduce noise by focusing on engagement with the recommended item. This helps because it isolates the effect of the specific recommendation from the myriad of other factors that might influence a user's overall listening behavior. By concentrating on the recommended item, these methods filter out the ``noise'' of unrelated listening activities, potentially leading to more precise estimates of the recommendation's impact.

Our approach goes a step further by leveraging other datasets to estimate long-term engagement following a discovery, including data from discoveries that occurred across the platform over a long time-span. It is worth emphasizing that this section demonstrates the sample efficiency benefits of the assumptions underlying our methods; the real-world justification for the assumptions comes from demonstrated impact in A/B tests, including on overall user behavior across the entire platform.

\subsubsection{Findings} Figure~\ref{fig:tars-sample-complexity} shows $1\sigma$ confidence intervals for each of the three value estimates, for increasing sample size $n$. Despite our experience across multiple A/B tests indicating that meta-actions $A$ and $B$ have significantly different long-term performance (see e.g. Figure \ref{fig:discovery-home}), these differences are statistically indistinguishable under the ``Holistic, long-term'' $Q$ estimation, even with a sample size of $n=2,560,000$. Strikingly, the standard error of the holistic, long-term $Q$ estimation is over 350 times larger than our approach. This implies that {\bf more than 120,000 times as much data} would be required to statistically distinguish a difference in $Q$-values of a given size using the ``Holistic, long-term'' method compared to our approach.

The ``Local, long-term'' estimator also dramatically outperforms the ``Holistic, long-term'' estimator in terms of standard error. This highlights the substantial benefits of focusing on a recommendation's contribution to listening habits formed with the specific recommended item, rather than trying to detect its impact on overall platform usage.

Notice that our estimator's standard  error even further the ``Local, long-term''. That is because it leverages auxiliary data to estimate the stickiness of discoveries. This improvement is particularly impactful when the number of randomized recommendations $n$ is relatively small; this scenario occurs naturally for less popular or niche content that doesn't receive frequent recommendations.

\begin{figure}[t]
    \centering
    \includegraphics{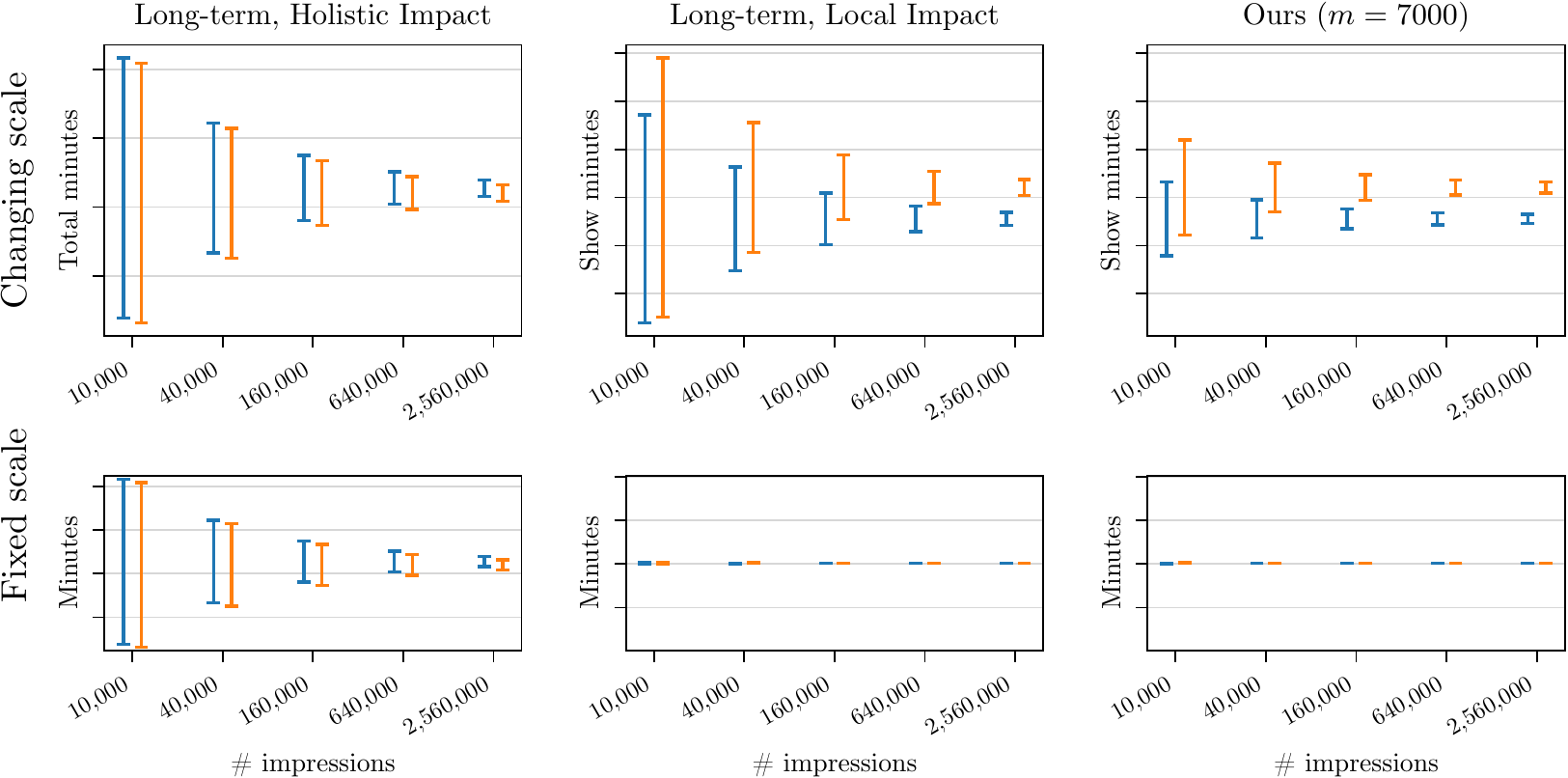}
    \caption{$1 \sigma$ confidence intervals for four different value estimates, as a function of sample size.
First row: $y$-axis scale varies for each plot.
Second row: $y$-axis scale is held constant across all three plots. That shows the standard errors of ``Holistic, long-term'' method are so enormous that everything else is negligible by comparison. }
    \label{fig:tars-sample-complexity}
\end{figure}

\section{Conclusion}

We have successfully optimized a component of an industrial-scale recommender system at Spotify for outcomes that occur over months, generating substantial impact even from a small alteration to the overall system. In this section, we offer some closing thoughts about some broader learnings that can be drawn from this experience.

To recommender systems practitioners, a field where myopic optimization is often the norm, our work offers tantalizing evidence that successful long-term optimization would result in large performance gains for many recommender systems, and perhaps digital platforms more broadly. We hope this spurs further energy in this direction.

To reinforcement learning researchers, who are focused anyway on the goal of optimizing for long-term outcomes, our work offers different learnings. Progress in this field is often driven by success in empirical benchmarks, and many benchmark problems are based on video games or robot simulators, where actions have large, easily detectable impacts. Where RL is applied in recommender systems, it is often used to optimize a very short session of interactions. For long-running user-interactions, measurement challenges become crippling. It's not that actions don't matter---large, persistent changes to policies still have substantial impact, but individual actions have mostly localized effects that are easily obscured by ``noise''. We managed to measure these localized effects through a focus on what we call ``item-level listening habits'', a solution built upon substantial domain-knowledge. A challenge for RL researchers is to develop algorithms that can automatically uncover solutions similar to our domain-specific approach.

\subsection*{Acknowledgments}

For their contributions to this work, we wish to thank Sainath Adapa, Shubham Bansal, Sonia Bhaskar, Dave Bredesen, Lauren Campbell, Kamil Ciosek,  Zhenwen Dai, Tonia Danylenko, David Gustafsson,  Tony Jebara, Jon King,  Mounia Lalmas-Roelleke, Crystal Lin,  Shawn Lin, Dylan Linthorne, Alan Lu, Andrew Martin, Roberto Sanchis Ojeda, Vladan Radosavljevic,  Oguz Semerci, Oskar Stal,   Tiffany Wu, and Cindy Zhang. 
We thank anonymous reviewers of the Reinforcement Learning for Real Life Workshop at NeurIPS 2022 for their feedback.
Daniel Russo would also like to thank Joseph Cauteruccio, Ian Anderson,  David Murgatroyd, and Clay Gibson, whose thoughts on reinforcement learning strategy influenced his thinking.

{\small
\setlength{\bibsep}{4pt plus 0.4ex}
\bibliography{references}
}

\appendix
\section{Background: embedded representations learned by pretraining on surrogate tasks}\label{sec:background}

At the core of many recommender systems is a vector embedding of users and items.
A user vector at the start of day $t$, $u_t \in \mathbb{R}^d$, encodes useful information about a user and their tastes.
For an item $a$ (e.g. a particular podcast show), the item vector $\nu_{a} \in \mathbb{R}^d$ encodes information about which users are likely to engage with the item.
The dot product $\nu_a^\top u_t$ represents a user's ``affinity'' for item $a$.
We will think of $u_t$ as capturing the system's understanding of the user's overall tastes and a user-item affinity score captures the user's propensity to have a short-term engagement with that type of item.
A simple default recommendation strategy would be to display to the user the items with which they have greatest affinity.

At times in this paper, we will assume access to a trained representations of users and items. 
We do not describe the specific systems in use at Spotify. But we briefly overview an influential methodology described in \cite{covington2016deep}, providing some background on a deep-learning based user/item embedding procedures. To train a neural network that produces user and item vectors, they design a surrogate multi-class classification problem.
To produce a training example, they randomly  select one video the user watched on a given day from the set of videos watched.
The ID of the selected video is the label. The network takes information on the user's features and their interactions with the app as input and is trained to output probabilities (for each item in a restricted corpus) that minimize cross entropy loss in predicting the label.
The output of the network's last hidden layer is a $256$ dimensional user vector and the trained weights of the final layers are the $256$ dimensional item vectors.
One can interpret $\exp\{u_t^\top \nu_a\}$, the exponentiated affinity for item $a$, as being proportional to the predicted probability that $a$ is the correct label.
It roughly captures their relative likelihood of engaging with items that have similar representations in the future.

Figure \ref{fig:youtube} shows the network architecture in more detail.
As input, a collection of the id's of recent watched videos and search tokens.
An embedding layer transforms one-hot encodings of the video ids and search tokens into dense vectors, which are then averaged, producing the ``watch vector'' and ``search vector'' in Figure \ref{fig:youtube}.
This is concatenated with other user features, and transformed through three hidden layers.
Because of the averaging, we expect that user vectors change slowly if a long history of watches and searches is provided as input. In this case, we think of user vectors as encoding their long-running tastes.

\begin{figure}[t]
    \centering
    \setlength{\fboxsep}{1pt}%
    \setlength{\fboxrule}{1pt}%
    \fbox{\includegraphics[width=.5\textwidth]{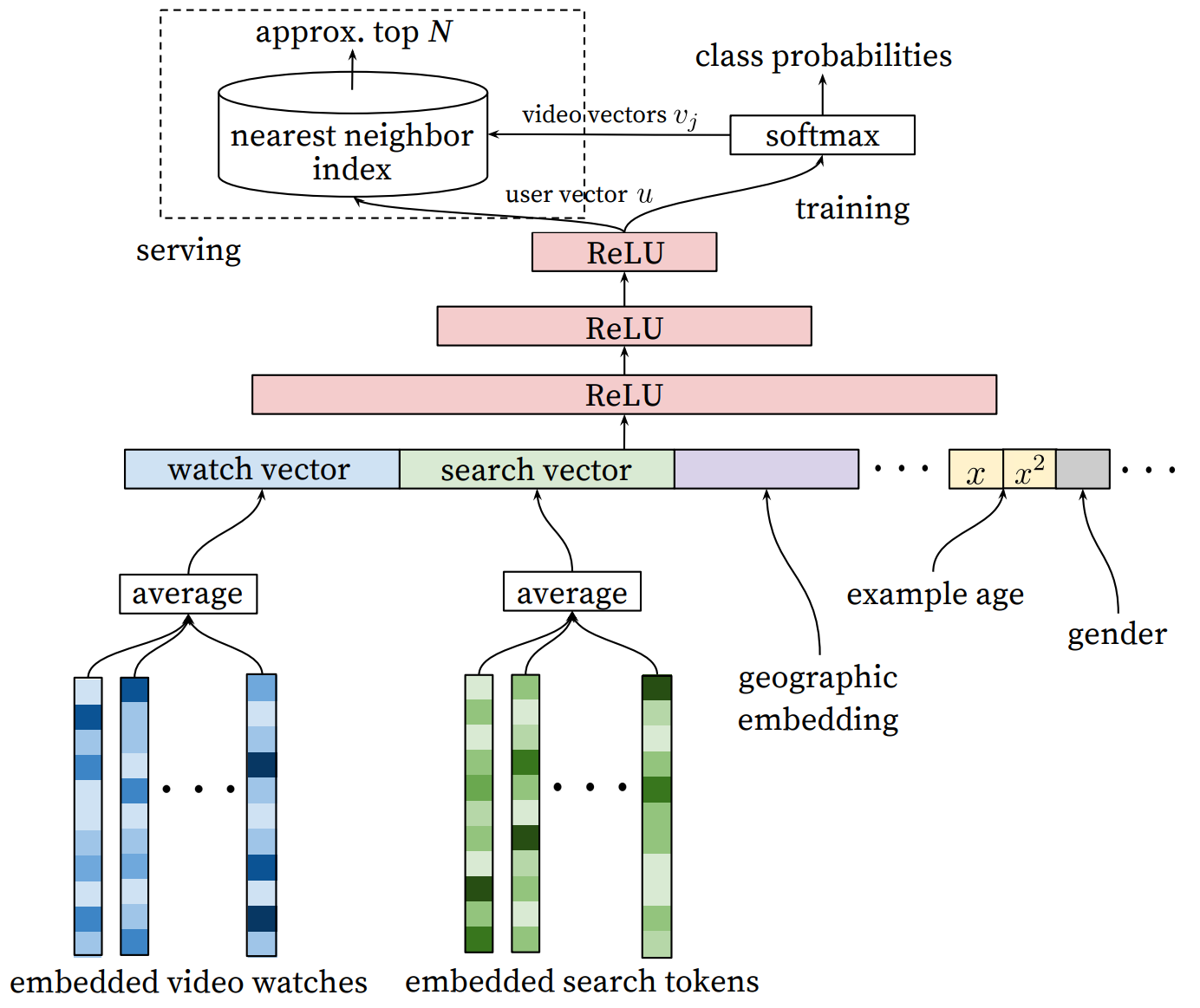}}
    \caption{
        Figure appearing in \cite{covington2016deep} describing a system for candidate generation at YouTube.}
        \label{fig:youtube}
\end{figure}

\section{Proof of Theorem \ref{thm:main}}

\begin{proof}
    Without loss of generality, we take $t = 0$.
    Using first using the separable form the reward in \eqref{eq:separable-rewards}, then Assumption \ref{assumption:direct-short-term-impact}, we find
    \begin{align*}
        Q_{\pi^0}(S_{0}, a)
        &= \mathbb{E}_{\pi^0}\left[  \sum_{\tau=0}^{T_1}  R\left( Y_{\tau}, A_{\tau} \right) \mid A_{0, \star}=a ,\, S_{0} \right] \\
        &= \mathbb{E}_{\pi^0}\left[  \sum_{\tau=0}^{T_1}\sum_{a'\in \Abb}  r(C_{\tau,a'})      \mid A_{0, \star}=a ,\, S_{0} \right] \\
        &= \sum_{a'\in \Abb}  \mathbb{E}_{\pi^0}\left[  \sum_{\tau=0}^{T_1}   r(C_{\tau,a'})      \mid A_{0, \star}=a ,\, S_{0} \right] \\
        &= \underbrace{\sum_{a'\in \Abb_{\star}}  \mathbb{E}_{\pi^0}\left[  \sum_{\tau=0}^{T_1}  r(C_{\tau,a'})      \mid A_{0, \star} \neq a' ,\, S_{0} \right]}_{:= b_{\pi^0}(S_0) } \\
        &\quad +  \left(  \underbrace{\mathbb{E}_{\pi^0}\left[  \sum_{\tau=0}^{T_1}   r(C_{\tau,a})      \mid A_{0,\star} =  a  ,\, S_{0} \right]}_{ := (*)} - \underbrace{\mathbb{E}_{\pi^0}\left[  \sum_{\tau=0}^{T_1}   r(C_{\tau,a})      \mid A_{0, \star} \neq  a ,\, S_{0} \right]}_{:=(**)} \right).
    \end{align*}
    We now simplify the term (*) by applying Assumption \ref{assumption:markov-in-time}.
    We have,
    \begin{align*}
        &\ \mathbb{E}_{\pi^0}\left[  \sum_{\tau=0}^{T_1}    r(C_{\tau,a})      \mid A_{0, \star} =  a  ,\, S_{0} \right] \\
        =\quad& \mathbb{E}_{\pi^0}\left[ r(C_{0,a})    +  \ind(T_1\geq 1)\sum_{\tau=1}^{T_1}  r(C_{\tau,a})      \mid A_{0, \star} =  a,\, S_{0} \right] \\
        \overset{(a)}{=} \quad & \mathbb{E}_{\pi^0}\left[ r(C_{0,a})    +  \ind(T_1\geq 1) \mathbb{E}_{\pi^0}\left[  \sum_{\tau=1}^{T_1}    r(C_{\tau,a})  \mid Z_{1,a}, S_{0}, \ind(A_{0, \star}=a)=1 , \,  T_1 \geq  1 \right]     \mid A_{0,\star} =  a ,\, S_{0} \right] \\
        \overset{(b)}{=} \quad& \mathbb{E}_{\pi^0}\left[ r(C_{0,a})    +  \ind(T_1\geq 1) \mathbb{E}_{\pi^0}\left[  \sum_{\tau=1}^{T_1}    r(C_{\tau,a})  \mid Z_{1,a}, S_{0}  ,\,  T_1 \geq  1 \right]     \mid A_{0,\star} =  a ,\, S_{0} \right] \\
        \overset{(c)}{=} \quad& \mathbb{E}_{\pi^0}\left[ r(C_{0,a})    +  \ind(T_1\geq 1) \mathbb{E}_{\pi^0}\left[  \sum_{\tau=1}^{T_1}    r(C_{\tau,a})  \mid Z_{1,a}, S_{0} , \,  S_1 \neq \varnothing \right]     \mid A_{0,\star} =  a ,\, S_{0} \right] \\
        \overset{(d)}{=} \quad& \mathbb{E}_{\pi^0}\left[ r(C_{0,a})    +  \ind(T_1\geq 1) V^{(a)}_{\pi^0}\left( Z_{1}^{(a)} \,,\, S_0\right)   \mid A_{0,\star} =  a \,,\, S_{0} \right]   \\
        \overset{(e)}{=} \quad& \mathbb{E}_{\pi^0}\left[ r(C_{0,a})    +  \gamma V^{(a)}_{\pi^0}\left( f\left(Z_{0,a}, C_{0,a} \right)   \,,\, S_0\right)   \mid A_{0,\star} =  a \,,\, S_{0} \right].   \\
    \end{align*}
    Step (a) above uses the law of iterated expectations.
    Step (b) follows from Assumption \ref{assumption:markov-in-time}.
    Step (c) is just a notation change. The statement $T_1\geq 1$ is equivalent to the notation $S_1\neq \varnothing$.
    Step (d) applies the definition of the item-level value function.
    Step (e) applies the incremental definition of the content-state, $Z_{t+1,a}=f(Z_{t,a} \,,\, C_{t,a})$ together with the assumption that user lifetimes follow a geometric distribution that is drawn independently from the random shocks $(\epsilon_t)$ and $(\xi_{t})$ that determine all other variables  (See Sec.~\ref{sec:formulation}).

    Applying the same steps to simplify (**) yields
    \begin{align*}
        Q_{\pi^0}(S_0,a) = b_{\pi^0}(S_t) &+  \mathbb{E}_{\pi^0}\left[ r(C_{0,a})    +  \gamma V^{(a)}_{\pi^0}\left( f\left(Z_{0,a}, C_{0,a} \right)   \,,\, S_0\right)   \mid A_{0,\star} =  a,\, S_{0} \right] \\
        &-  \mathbb{E}_{\pi^0}\left[ r(C_{0,a})
        +  \gamma V^{(a)}_{\pi^0}\left( f\left(Z_{0,a}, C_{0,a} \right)   \,,\, S_0\right)   \mid A_{0,\star} \neq  a,\, S_{0} \right].
    \end{align*}
\end{proof}

\section{Details of the offline experiment in Section \ref{sec:resurfacing-eval}}\label{sec:resurfacing-eval-details}

We detail the estimation used in Section  \ref{sec:resurfacing-eval}.
Similarly to the discovery setting, we collect two datasets.
\begin{enumerate}
\item The first dataset $\Dscr = \{ A, u, Z_a, \ind(Y > 0) \}$ contains immediate outcomes to recommendations.

\item For every item $a$, we gather a second dataset $\Dscr_a = \{ u, Z_a, \ind(C_a > 0), R \}$ that contains user state information, and immediate and long-term organic consumption.
$\ind(C_a > 0)$ indicates a stream of item $a$ occurred on a given day, and $R \in \{0, \ldots,  59\}$ is the number of return days to the item, similarly to Section~\ref{subsubsec:discovery_stickiness}.
\end{enumerate}
For purposes of illustration, we train non-parametric models on an aggregated representation of the content-relationship state $Z$.
Letting $\Rscr(z)$ be the set of all states in the same aggregation region as $z$, we estimate
the probability that a user in state $z$ streams from a recommendation of content $a$ as
\begin{align*}
P(a, z)
    &= \mathbb{E}_{\Dscr} [ \ind(Y > 0) \mid A = a, Z \in \Rscr(z) ],
\end{align*}
the probability that a user in state $z$ streams content $a$ organically as
\begin{align*}
P(C_a > 0 \mid z, \text{no rec})
    &= \mathbb{E}_{\Dscr_a} [ \ind(C > 0) \mid A = a, Z \in \Rscr(z) ],
\end{align*}
and the expected return days of a user in engagement state $z$ as
\begin{align*}
V^{(a)}(z)
    &= \mathbb{E}_{\Dscr_a} [ R \mid A = a, Z \in \Rscr(z) ],
\end{align*}
where $\mathbb{E}_{\Dscr} [ \cdot ]$ denotes the empirical average over $\Dscr$.
Given these, we estimate the probability that a user in state $z$ streams content $a$ given a recommendation as
\begin{align}
\label{eq:resurfacing-prec}
P(C_a > 0 \mid z, \text{rec})
    &= P(a, z) + [1 - P(a, z)] \times P(C_a > 0 \mid z, \text{no rec}).
\end{align}
Observe that, while learning $P(C_a > 0 \mid z, \text{rec})$ requires data from a specific recommender system of interest, we are able to take advantage of data from across the entire application to learn a model for $V$. 

In our numerical illustrations, we use the approximations 
\begin{align*}
\mathbb{P}_{\pi^0}\left( P(C_{t,a} > 0 \mid Z_{t,a}=z, A_{t,\star}=a) \right)  &\approx P(C_a > 0 \mid z, \text{rec}) \\
\mathbb{P}_{\pi^0}\left( P(C_{t,a} > 0 \mid Z_{t,a}=z, A_{t,\star}\neq a) \right)  &\approx P(C_a > 0 \mid z, \text{no rec}).
\end{align*}


\section{Reviewing theory of policy improvement: connections to coordination and attribution challenges}
\label{sec:policy_improvement_theory}

This section reviews theory of a simple and direct approach to improving a component of the recommendation policy:
fit a parametric approximation to the $Q$-function in \eqref{eq:partialQ} and then adjust the recommendation policy toward actions with higher $Q$-values. 
Our methods follow this general template, but employ domain-specific, tailored, modeling of  the $Q$-function. We highlight that, even without our domain-specific modeling, this general approach to improving a specific policy component is an attractive one when faced with the coordination and attribution challenges described in Section \ref{subsec:challenges}. (As highlighted in the body of the paper, our structural of the $Q$-function is most critical for enabling pragmatic, sample-efficient, estimation of the $Q$-function.)

We do this by reviewing and specialize theory of  \emph{actor-critic} methods \citep{sutton1999policy,konda1999actor}.

\subsection{An approximate policy improvement update via logged data}
\label{subsec:approx_pi}
We derive rigorous insights in the case where $Q$-functions are approximated using state-aggregation, a special kind of parametric approximation.
Because a user's state is represented by the exhaustive history of their interactions, it is possible that no two users share an identical state.
A successful recommender system needs to recommend items to users that were liked by users with \emph{similar} states.
How does this fit with the goal of policy improvement described in the previous section?
Here we introduce the reader to ideas of approximation and generalization in policy improvement by considering state-aggregated policies, a simple form of approximation that has been studied for decades \citep{whitt1978approximations, bean1987aggregation, singh1995reinforcement, gordon1995stable, tsitsiklis1996feature,rust1997using, li2006towards, jiang2015abstraction}.

Under a state-aggregated policy, the state space is partitioned into segments, and users whose states fall into a common segment receive the same recommendation.
Fix a rule $\phi: \Sbb \to \{1,\cdots, m\}$ which assigns each state to one of $m$ segments.
We stay agnostic to the choice of $\phi$, determining how the state-space is partitioned, but one could imagine applying a standard procedure to cluster user-vectors, like those described in Appendix \ref{sec:background}.
We say $\phi$ is \emph{non-degenerate} if there are no empty user clusters, i.e. $\mathbb{E}_{\pi^0}\left[  \sum_{t=T_0}^{T_1} \ind\{ \phi(S_t)=i \} \right]>0$ for each $i\in[m]$.

For a given non-degenerate aggregation rule $\phi:\Sbb \to [m]$, define the state-aggregated value function $\bar{Q}_{\pi^0}: [m] \times \Abb_\star$
by
\begin{equation}\label{eq:aggregated-Q}
    \bar{Q}_{\pi^0}(i, a) = \frac{\mathbb{E}_{\pi^0}\left[  \sum_{t=T_0}^{T_1}  Q_{\pi^0}(S_t, a ) \ind\{ \phi(S_t)=i \} \right] }{ \mathbb{E}_{\pi^0} \left[  \sum_{t=T_0}^{T_1} \ind\{ \phi(S_t)=i \} \right]  } \quad \forall i\in[m], a\in \Abb_{\star}.
\end{equation}
This can be thought of as a regression-based approximation to the $Q$-function.
Among all state-aggregated functions, $\bar{Q}(\cdot, a):\Sbb \to \mathbb{R}$ minimizes the mean-squared error metric $\mathbb{E}_{s\sim w}( \bar{Q}_{\pi^0}( \phi(s), a) -   Q_{\pi^0}(s,a) )^2 $ where $w(s)\propto \mathbb{E}_{\pi^0}\left[\sum_{t=T_0}^{T_1} \ind\{S_t=s\}\right]$. See Appendix \ref{subsec:state_agg_as_projection} for details.

Whereas the exact policy iteration update requires representing the value of each action at each possible state, forming the approximation in \eqref{eq:aggregated-Q} requires computing $|\Abb_{\star}| \cdot m$ averages.
Crucially, the averages are computed under the distribution of states visited by the incumbent policy, which is exactly the data we have access to in $\mathcal{D}$.
Algorithm \ref{alg:direct} shows how to build a sample based approximation to this policy improvement problem, based on the definition of the $Q$-functions in \eqref{eq:partialQ}.
\begin{algorithm}
    \caption{Direct, locally optimal, state-aggregated policy improvement}\label{alg:direct}
    \begin{algorithmic}
        \Require Aggregation rule $\phi: \Sbb \to [m]$, dataset of user trajectories $\{(S_t^u, A_t^u, Y_t^u) : t\in \{T_0^u, \cdots, T_1^u\}\}_{u \in U}$ , reward function $R$.
        \For{$i \in [m]$}
        \For{$a \in \Abb_{\star}$}
        \State $\Dscr_{i,a} = \{ (u,t) : \phi(S_t^u) = i\, ,\, A_{t,\star}^u = a \}$.
        \State $\hat{Q}^{\rm direct}\left(i, a\right)   = \frac{1}{|\Dscr_{i,a}|} \sum_{(u,t)\in \Dscr_{i,a}} \left( \sum_{\tau=t}^{T_1^u} R(A_{\tau}^u, Y_{\tau}^u)  \right)$.
        \EndFor
        \State Pick $\hat{a}_i \in \argmax_{a \in \Abb_{\star}} \hat{Q}^{\rm direct}(i, a)$
        \EndFor
        Return  $\hat{Q}^{\rm direct}$ and $(\hat{a}_1, \cdots, \hat{a}_m)$  defining a policy.
    \end{algorithmic}
\end{algorithm}

\subsection{How coordination challenges are addressed}
\label{subsec:coordination_theory}

Classical theory of dynamic programming covers policy iteration methods that optimize the true $Q$-function at every state.
This subsection interprets methods that optimize the simplified approximate $Q$-function in  \eqref{eq:aggregated-Q} as performing \emph{coordinate ascent}:
they produce a steepest ascent update to a component of the recommendation policy, within a restricted class of policies.
This is accomplished naturally by fitting $Q$-functions on data generated by the incumbent policy.

Coordinate ascent is a \emph{decentralized} method if not a \emph{coordinated} one.
It optimizes a component of the recommendation system in the context of how the rest of the system behaves.
A coordinated solution might instead change many parts of the system to attain a goal.
To appreciate this distinction, consider our podcast discovery prototypes from Section \ref{sec:prototypes}.
Those methods predict whether a user is likely to form a listening habit with a podcast show if the try it once.
Implicitly, this depends not just on the user's tastes, but on the incumbent policy of the recommender system, which influences whether subsequent recommendations resurface that show to the user.

The rest of this subsection makes the coordinate ascent interpretation formal.
First, the definition below formally defines a set of deterministic state-aggregated policies which may deviate from the incumbent policy at position $\star$ but not at other positions.
\begin{definition}[]\label{def:state-agg-policy}
    A state aggregation rule is a map $\phi: \Sbb \to [m]$. For a given aggregation rule $\phi$, take $\Pi_{\star}^{\phi}$ to be the set of all policies of the form $\pi=(\pi_\star, \pi_{\smallsetminus \star})$ such that  $\pi_{\star}(s, \epsilon_{t,\star})$ does not depend on the idiosyncratic randomness $\epsilon_{t,\star}$ and depends on $s$ only through $\phi(s)$.  Each such policy can be represented by $m$ actions $a_1, \cdots, a_m \in \Abb_{\star}$ with $\pi_{\star}(s, \epsilon_{t,\star})=a_{\phi(s)}$ almost surely for each $s\in \Sbb$.
\end{definition}

Recall the performance measure $J(\pi)=\E[V_{\pi}(S_{T_0})]$ represents the average cumulative reward across a user's lifetime interactions. For a policy $\pi' \in \Pi_{\star}^{\phi}$, take ${\rm Mix}(\pi', \pi_0,  \beta)$ to be a policy that,
independently in each period, selects the action $\pi'_{\star}(S_t, \epsilon_{t,\star})$ with probability $\beta$ and otherwise selects $\pi^0_{\star}(S_t, \epsilon_{t,\star})$. Consider the local policy improvement problem,
\begin{equation}\label{eq:aggregated-local-pi}
\pi^+ \in \argmax_{\pi \in \Pi_{\star}^{\phi}}    \, \frac{d}{d\beta} \, J\left(  \, {\rm Mix} (\pi, \pi_0,  \beta) \,   \right) \bigg\vert_{\beta=0},
\end{equation}
which aims to attain the steepest rate of improvement in lifetime reward among all state-aggregated policies.
The randomization in ${\rm Mix}(\pi', \pi_0,  \beta)$ is not an essential part of the current paper.
One should think of this a formal device for studying \emph{small} alterations of the user experience.\footnote{%
The theory of policy gradient methods clarifies when the derivative with respect to $\beta$ accurately captures also the change for moderate values of $\beta$.
The major concern is distribution shift, where altering component $\pi_{\star}$ substantially alters the fraction of users whose states fall within various clusters.
See the discussion in \cite{schulman2015trust}, or, for precise expressions, see the proof in Appendix \ref{subsec:proof_of_pg_formula}.}
The next lemma shows that $\pi^+$ is the maximizer of an approximate $Q$-function. This result is a modified version of the ``actor-critic'' form of the policy gradient theorem, first discovered by \cite{sutton1999policy} and \cite{konda1999actor}.
Theorem~8 in \cite{russo2020approximation} provides a result that is almost analogous to this one.
For completeness, we provide a proof in Appendix \ref{subsec:proof_of_pg_result}.
\begin{lemma}\label{lem:policy_grad}
    For a given non-degenerate aggregation rule $\phi:\Sbb \to [m]$, consider the state-aggregated value function $\bar{Q}_{\pi^0}: [m] \times \Abb_\star$ defined in \eqref{eq:aggregated-Q}. A policy $\pi^+\in \Pi_{\star}^{\phi}$ represented by actions $(a^{(1)}, \cdots, a^{(m)})$ solves \eqref{eq:aggregated-local-pi} if and only if
    \begin{equation}\label{eq:greedy_actions}
     a^{(i)} \in \argmax_{a \in \Abb_{\star}}  \bar{Q}_{\pi^0}(i, a)   \quad \text{for each } i\in [m].
    \end{equation}
\end{lemma}

\subsection{How attribution challenges are addressed}
\label{subsec:attribution_theory}

Attribution challenges are also addressed coherently.
Credit for selecting action $\pi'_{\star}(S_{t}, \epsilon_t)$ rather than  the prescribed action $\pi^0_{\star}(S_{t}, \epsilon_t)$ is assigned based on the difference in fitted $Q$-values under those actions.
This difference is often called an \emph{advantage} in RL.
The next lemma shows that, in expectation, the cumulative sum of credited values equals the true difference in long-term value, up to a second order term.
This property is similar to what \cite{singal2022shapley} call \emph{counterfactual efficiency} in their axiomatic approach to coherent attribution.
The term that is $O(\beta^2)$ depends on the degree of distribution shift and is described explicitly in Appendix \ref{subsec:proof_of_pg_formula}.
\begin{lemma}[Specialization of the policy gradient theorem]\label{lem:pg-formula} For a policy $\pi' \in \Pi^\phi_{\star}$,
 \[
 \underbrace{J( {\rm Mix}(\pi', \pi^0, \beta) ) - J(\pi^0)}_{\text{Change in long-term value}}
 = \beta \mathbb{E}_{\pi^0}\left[ \sum_{t=T_0}^{T_1}  \underbrace{\left( \bar{Q}_{\pi^0}\left( \phi\left(S_{t}\right) ,   \pi'_{\star}(S_{t}, \epsilon_{t,\star}) \right) - \bar{Q}_{\pi^0}\left( \phi\left(S_{t}\right) ,   \pi^0_{\star}(S_{t}, \epsilon_{t,\star}) \right)   \right)}_{ \text{Estimated advantage of } \pi' \text{ over } \pi^0 \text{ at } S_t } \right] + O(\beta^2).
 \]
\end{lemma}

\subsection{Proofs of gradient results}

For completeness, this section provides an independent derivation of results on actor-critic algorithms state in Lemmas \ref{lem:policy_grad} and \ref{lem:pg-formula}.
The section needs to be read linearly.
Each subsection provides new notations and results that are used in subsequent ones.

\subsubsection{State occupancy measure}

Define the state occupancy measure under policy $\pi$ by
\[
\eta_{\pi}\left( \tilde{\Sbb} \right) = c  \cdot\mathbb{E}_{\pi}\left[ \sum_{t=T_0}^{T_1} \ind \left( S_t \in \tilde{\Sbb} \right)\right] \qquad \forall  \, \tilde{\Sbb} \subset \mathbb{\Sbb}
\]
where $c=1/(1-\gamma)$ is the appropriate normalizing constant which ensures $\eta_{\pi}(\Sbb)=1$.

\subsubsection{Interpretation of state-aggregation as orthogonal projection}\label{subsec:state_agg_as_projection}
Define a weighted the inner product on $\mathbb{R}^{\Sbb \times |\Abb_{\star}| }$ by
\[
\langle Q,  Q'\rangle_{\pi} = \sum_{a\in \Abb_{\star}}\E_{s\sim \eta_{\pi}}\left[ Q(s,a) Q'(s,a)\right].
\]
Take $\| Q \|_{\pi} = \sqrt{ \langle Q,  Q\rangle_{\pi} }$ to be the associated Euclidean norm.
Define the subspace of state aggregated $Q$ functions by
\[
\mathbb{Q}_{\phi} = \{ Q :\Sbb\times \Abb_{\star} \to \mathbb{R} \mid  Q(s,a)=Q(s',a)  \quad  \forall s,s' \text{such that }   \phi(s) = \phi(s')  \}.
\]
The projected state aggregated $Q$-function
\[
\hat{Q}_{\pi^0}^{\phi} = \argmin_{Q \in \mathbb{Q}_{\phi}} \| Q - Q_{\pi^0}\|_{\pi^0},
\]
satisfies
\[
\hat{Q}_{\pi^0}(s, a) = \mathbb{E}_{s' \sim \eta_{\pi^0}}\left[ Q_{\pi^{0}}(s',a) \mid \phi(s') = i \right]  \equiv \bar{Q}_{\pi^0}(i,a)  \qquad \forall i \in [m], a\in \Abb_{\star}.
\]
In words, predicting the conditional mean minimizes mean-squared error. The notation $\bar{Q}_{\pi^0}(i,a)$  is defined in \eqref{eq:aggregated-Q}. It encodes the same values as the projected $Q$-function, $\hat{Q}_{\pi^0}$, but takes as input the index of a state cluster rather than a state.

\subsubsection{A specialized performance difference lemma}
\label{subsec:perf_diff}

We state a variant of the performance difference lemma \citep{kakade2002approximately} that applies when only a single component of the policy is changed.
Define
\[
Q_{\pi^0}^{\rm Full}(s, [a_{1},\ldots, a_{L}]) =  \mathbb{E}_{\pi^0}\left[  \sum_{t=0}^{T_1}  R(Y_t, A_t)  \mid S_0 = s, A_{0} =[a_{1},\ldots a_L] \right]
\]
and recall
\[
Q_{\pi^0}(s, a)  = \mathbb{E}_{\pi^0}\left[  \sum_{t=0}^{T_1}  R(Y_t, A_t)  \mid S_0 = s, A_{0,\star} =a\right].
\]
The performance difference lemma states
\[
J(\pi) - J(\pi_0) = \E\left[ Q^{\rm Full}_{\pi^0}\left(S, \pi(S, \epsilon_t) \right) -  Q^{\rm Full}_{\pi^0}\left(S, \pi^0(S, \epsilon_t) \right)  \right] \qquad S\sim \eta_{\pi}
\]
where the expectation is over $S$ drawn from the occupancy measure $\eta_{\pi}$ and the action noise $\epsilon_t$. A policy $\pi \in \Pi_{\star}^{\phi}$ can be written as $\pi=(\pi_{\star}, \pi_2^0, \ldots, \pi_L^0)$, differing from the incumbent policy $\pi^0$ only on component $\star$.
(See Definition \ref{def:state-agg-policy}.)
We find that for all $\pi \in \Pi_{\star}^{\phi}$,
\begin{align*}
J(\pi) - J(\pi_0) &= \E\left[ Q^{\rm Full}_{\pi^0}\left(S, [\pi_{\star}(S, \epsilon_{t,\star}), \pi_{\setminus \star}(S, \epsilon_{t,\setminus \star})] \right) -  Q^{\rm Full}_{\pi^0}\left(S, [\pi^0_{\star}(S, \epsilon_{t,\star}), \pi^0_{\setminus \star}(S, \epsilon_{t,\setminus \star})]\right)  \right] \\
&=\E\left[ Q^{\rm Full}_{\pi^0}\left(S, [\pi_{\star}(S, \epsilon_{t,\star}), \pi^0_{\setminus \star}(S, \epsilon_{t,\setminus \star})] \right) -  Q^{\rm Full}_{\pi^0}\left(S, [\pi^0_{\star}(S, \epsilon_{t,\star}), \pi^0_{\setminus \star}(S, \epsilon_{t,\setminus \star})]\right)  \right] \\
&= \E\left[ Q_{\pi^0}\left(S, \pi_{\star}(S, \epsilon_{t,\star}) \right) -  Q_{\pi^0}\left(S, \pi^0_{\star}(S, \epsilon_{t,\star}) \right)  \right]
\end{align*}
where $S\sim \eta_{\pi}$ is assumed to be drawn independently of $\epsilon_t$. The second equality above uses that policies are state aggregated. The first and third apply the definitions of $Q^{\rm Full}_{\pi^0}$ and $Q_{\pi^0}$.
With some abuse of notation, we now interpret $\pi_{\star}(s, \cdot)$ as specifying a probability vector $\left(\pi_{\star}(s,a) : a \in\Abb_{\star}\right)$ with elements
\[
\pi_{\star}\left(s,a \right) = \Prob( \pi_{\star}(s,\epsilon_{t,\star})=a).
\]
The performance difference lemma can be rewritten as
\begin{equation}\label{eq:perf-diff}
J(\pi) - J(\pi^0) = \E_{S \sim \eta_\pi}\left[ \sum_{a \in \Abb_\star} Q_{\pi^0}\left(S, a \right)\left( \pi_{\star}(S,a)- \pi^0_{\star}(S,a) \right)   \right]  = \langle Q_{\pi^0} \, ,\, \pi_{\star} - \pi^0_{\star} \rangle_{\pi}
\end{equation}
for all $\pi \in \Pi_{\star}^{\phi}$. The inner product was defined in Subsection \ref{subsec:state_agg_as_projection}.

\subsubsection{Proof of Lemma \ref{lem:pg-formula}}\label{subsec:proof_of_pg_formula}
The result we wish to show is an actor-critic form of the policy gradient theorem, similar to the one derived by \citet{sutton1999policy} and \citet{konda1999actor}.
The result here is in a slightly different form however, which is specialized to the case of state-aggregated policies a particular (so called ``direct'') parameterization.
\citet{russo2020approximation} establishes almost the same result in the case where the state space is finite.

\begin{proof}[Proof of Lemma \ref{lem:pg-formula}]
Consider the randomized mixture policy $\pi^\beta = {\rm Mix} (\pi, \pi_0,  \beta)$, which can be written $\pi^\beta = \beta \pi +  (1-\beta) \pi^0$.
Then, by \eqref{eq:perf-diff},
\begin{align*}
J(\pi^\beta) - J(\pi^0) =& \beta  \cdot \langle Q_{\pi^0} \, ,\, \pi_{\star} - \pi^0_{\star} \rangle_{\pi^{\beta}} = \beta \cdot  \langle Q_{\pi^0} \,, \, \pi_{\star} - \pi^0_\star \rangle_{\pi^0} + \mathcal{E}(\pi^0, \pi^\beta),
\end{align*}
where the remainder term is defined by
\[
\mathcal{E}(\pi^0, \pi^\beta) =  \beta \left[  \langle Q_{\pi^0} \, ,\, \pi_{\star} - \pi^0_{\star} \rangle_{\pi^{\beta}}  - \langle Q_{\pi^0} \,, \, \pi_{\star} - \pi^0_\star \rangle_{\pi^0} \right] = O(\beta^2).
\]
We return later to justify that this term is $O(\beta^2)$.

Finally, we conclude
\begin{align*}
J(\pi^\beta) - J(\pi^0) &=   \beta \cdot  \langle Q_{\pi^0} \,, \, \pi_{\star} - \pi^0_\star \rangle_{\pi^0} +  \mathcal{E}(\pi^0, \pi^\beta) \\
&=\beta \cdot  \langle \hat{Q}^{\phi}_{\pi^0} \,, \, \pi_{\star} - \pi^0_\star \rangle_{\pi^0}  +  \beta \cdot \underbrace{\langle  Q_{\pi^0}-\hat{Q}^{\phi}_{\pi^0} \,, \, \pi_{\star}- \pi^0_\star \rangle_{\pi^0}}_{=0}   + \underbrace{\mathcal{E}(\pi^0, \pi^\beta)}_{O(\beta^2)},
\end{align*}
where the $\langle  Q_{\pi^0}-\hat{Q}^{\phi}_{\pi^0} \,, \, \pi_{\star}- \pi^0_\star \rangle_{\pi^0}=0$ follows the basic optimality conditions of orthogonal projection: the error vector $Q_{\pi^0}-\hat{Q}^{\phi}_{\pi^0}$ is orthogonal to the subspace $\mathbb{Q}_{\phi}$ in the inner product $\langle \cdot \, , \, \cdot \rangle_{\pi^0}$. The fact follows since the policies are state aggregated, meaning $\pi_{\star},\pi^0_\star \in \mathbb{Q}_{\phi}$.
This establishes Lemma \ref{lem:pg-formula}, albeit in different notation.

Let us return to sketch a proof that $\mathcal{E}(\pi^0, \pi^\beta) =O(\beta^2)$. To simplify, use the more abstract notation $Q= Q_{\pi^0}$ and $Q'= \pi_{\star} - \pi^0_\star$. Observe that both are \emph{bounded} functions mapping $\Sbb \times \Abb_{\star}$ to real numbers. Then
\begin{align*}
    \left|\mathcal{E}(\pi^0, \pi^\beta) \right|&= \beta \left|\langle Q\, ,\, Q' \rangle_{\pi^{\beta}}  - \langle Q \,, \, Q' \rangle_{\pi^0} \right| \\
    &= \beta \left| \sum_{a \in \Abb_\star} \left[  \mathbb{E}_{s\sim \eta_{\pi^\beta}}[ Q(s,a)Q'(s,a) ]  -  \mathbb{E}_{s\sim \eta_{\pi^\beta}}[ Q(s,a)Q'(s,a) ]    \right] \right|  \\
    & \leq \beta \|Q\|_{\infty}\|Q'\|_{\infty} |\Abb_\star| \delta( \eta_{\pi^\beta}, \eta_{\pi^0} )
\end{align*}
where $\delta(\cdot\, ,\, \cdot)$ denotes total variation distance. It is not difficult to show that $\delta( \eta_{\pi^\beta}, \eta_{\pi^0} ) = O(\beta)$. We only sketch the argument rather than develop new notation required to show things in formal math. Independently in each period, $\pi^\beta$ follows the action prescribed by $\pi^0$ with probability $1-\beta$. With probability $\beta$, it may select a different action. The probability $\pi^{\beta}$ ever prescribes a different action than $\pi^0$  during an episode is bounded by  the expected \emph{number} of times it prescribes a different action during an episode. The latter is equal to $\beta \E\left[ T_1 -T_0 \right] = O(\beta)$. Since, the probability $\pi^{\beta}$ ever prescribes a different action from $\pi^0$ is only $O(\beta)$, state trajectories under $\pi^\beta$ are identical to those under $\pi^0$ except with probability that is $O(\beta)$.
\end{proof}

\subsubsection{Proof of Lemma \ref{lem:policy_grad}}\label{subsec:proof_of_pg_result}
This result follows from Lemma \ref{lem:pg-formula} and some rewriting. We have
\begin{align*}
    J(\pi^\beta) - J(\pi^0) &= \beta \cdot  \langle \hat{Q}^{\phi}_{\pi^0} \,, \, \pi_{\star} - \pi^0_\star \rangle_{\pi^0}  + O(\beta^2) \\
    &= \beta \cdot \langle \hat{Q}^{\phi}_{\pi^0} \,, \, \pi_{\star} \rangle_{\pi^0} + \underbrace{\beta \cdot \langle \hat{Q}^{\phi}_{\pi^0} \,, \, \pi_{\star}^0 \rangle_{\pi^0}}_{\text{indep. of } \pi_\star} + O(\beta^2)
\end{align*}
Now, take $(a_1, \ldots a_m)$ to be the actions defining the state-aggregated policy $\pi_{\star}$, as in Definition \ref{def:state-agg-policy}. For any state $i$ in the $i^{\rm th}$ cluster (i.e. $\phi(s)=1$), one has $\hat{Q}^\phi_{\pi^0}(s,a) = \bar{Q}_{\pi^0}(i,a)$ and $\pi_{\star}(s,a_i)=1$, implying
\[
\langle \hat{Q}^{\phi}_{\pi^0} \,, \, \pi_{\star} \rangle_{\pi^0} = \beta \sum_{i=1}^{m} \eta_{\pi^0}(\{ s: \phi(s)=i \})  \bar{Q}_{\pi^0}(i, a_i).
\]
Therefore,
\[
\frac{d}{d\beta} J(\pi^\beta) \bigg\vert_{\beta=0}=  \sum_{i=1}^{m} \eta_{\pi^0}(\{ s: \phi(s)=i \})  \bar{Q}_{\pi^0}(i, a_i) + {\rm const}
\]
where ${\rm const}$ is independent of the choice of $(a_1, \cdots, a_m)$.
The steepest ascent direction in the space of state-aggregated policies is defined by the $m$ actions that maximize the aggregated $Q$-values as in \eqref{eq:greedy_actions}.

\end{document}